\documentclass[preprint,12pt]{elsarticle}

\usepackage{amssymb}
\usepackage{graphicx,xcolor}

\usepackage{pifont}
\usepackage{tabu}
\usepackage{booktabs}       
\usepackage{amsfonts}       
\usepackage{nicefrac}       
\usepackage{microtype}      
\usepackage{amsmath}
\usepackage{mathtools}
\usepackage{latexsym}
\usepackage{multirow}
\usepackage{xspace}
\usepackage{tabularx}

\usepackage{url}
\usepackage{color}
\usepackage{amsmath, amssymb}
\usepackage{type1cm}
\usepackage{comment}
\usepackage{amsmath}
\usepackage{amsfonts}
\usepackage{bm}
\usepackage{mathtools}
\usepackage{tikz}

\usepackage{colortbl}

\usetikzlibrary{shapes,positioning,calc,arrows,automata,intersections,matrix,decorations}
\usepgflibrary{decorations.pathmorphing}

\usepackage[left]{lineno}

\newcommand{\argmin}{\mathop{\rm argmin}\limits}

\newcommand{\Rset}{\mathbb{R}}
\newcommand{\mat}[1]{\boldsymbol{\mathbf{#1}}}

\newcommand{\sign}{\mathop{\text{sign}}}
\newcommand{\transpose}{^{\mathrm{T}}}
\newcommand{\binary}{^{(\text{b})}}
\newcommand{\diag}{\mathop{\text{diag}}}

\let\mytablefont\scriptsize

\definecolor{kblue}{HTML}{00AFEC}
\definecolor{kgreen}{HTML}{AACF52}
\definecolor{kyellow}{HTML}{FFF67F}
\definecolor{kred}{HTML}{FF9D00}
\definecolor{green}{HTML}{008000}

\usepackage{tikz}
\usetikzlibrary{calc,arrows,automata,positioning,fit} 

\newcommand{\tikznode}[2]{\relax
\ifmmode%
  \tikz[remember picture,baseline=(#1.base),inner sep=0pt] \node (#1) {$
  #2$};
\else
  \tikz[remember picture,baseline=(#1.base),inner sep=0pt] \node (#1) {#2};%
\fi}

\usepackage{blkarray}

\newcommand*{\colorboxed}{}
\def\colorboxed#1#{%
  \colorboxedAux{#1}%
}
\newcommand*{\colorboxedAux}[3]{%
  \begingroup
    \colorlet{cb@saved}{.}%
    \color#1{#2}%
    \boxed{%
      \color{cb@saved}%
      #3%
    }%
  \endgroup
}
\newcommand{\transposeb}{^{(b)\mathrm{T}}}
\usepackage{amsthm}
\theoremstyle{definition}
\newtheorem{definition}{Definition}
\newtheorem{lemma}{Lemma}
\newtheorem{corollary}{Corollary}
\newtheorem{thm}{Theorem}

\usepackage{enumitem}

\journal{Knowledge-Based Systems}

\begin{document}

\begin{frontmatter}



 \title{Binarized Canonical Polyadic Decomposition for Knowledge Graph Completion\tnoteref{ecir}}



\author[label1]{Koki Kishimoto\corref{cor1}\fnref{fn1}}
\ead{kishimoto@ei.sanken.osaka-u.ac.jp}

\author[label2,label4]{Katsuhiko Hayashi\fnref{fn1}}

\author[label1]{Genki Akai}

\author[label3,label4]{Masashi Shimbo}

\cortext[cor1]{Corresponding author}
	\address[label1]{Osaka University, Osaka, Japan}
	\address[label2]{The University of Tokyo, Tokyo, Japan}
	\address[label3]{Nara Institute of Science and Technology, Nara, Japan}
	\address[label4]{RIKEN Center for Advanced Intelligence Project}

 \tnotetext[ecir]{
   This paper extends the
   conference paper \cite{bcp} that appeared in European Conference on Information Retrival '19,
   mainly by adding the proof of full expressiveness of B-CP and 
   experiments on additional datasets.
}
\fntext[fn1]{These authors contributed equally to this work.}

\begin{abstract}
Methods based on vector embeddings of knowledge graphs have been actively pursued as a promising approach to knowledge graph completion.
However, embedding models generate storage-inefficient representations, particularly when the number of entities and relations, and
the dimensionality of the real-valued embedding vectors are large.
We present a binarized CANDECOMP/PARAFAC~(CP) 
decomposition algorithm, which we refer to as B-CP,
where real-valued parameters are replaced by binary values to reduce
model size.
Moreover, we show that a fast score computation technique can be
developed with bitwise operations.
We prove that B-CP is fully expressive by deriving a bound on the size
of its embeddings.
Experimental results on several benchmark datasets
demonstrate that the proposed method
successfully reduces model size by more than
an order of magnitude while maintaining task performance at the same level as the real-valued CP model.


\end{abstract}




\end{frontmatter}


\section{Introduction}
\label{intro}
\begin{figure}[t]
\centering
\begin{tikzpicture}[xscale=3.5,yscale=3.4]
  \tikzstyle{every state}=[minimum size=2mm]

  \foreach \x/\y/\name/\dispname/\pos in {%
    0/0/actor1/Leonard Nimoy/below,
    1/0/movie1/Star Trek/below,
    0/1/char1/Spock/above,
    1.5/1/genre/SciFi/above,
    2/0/movie2/Star Wars/below,
    3/0/actor2/Alec Guinness/below,
    3/1/char2/Obi-Wan Kenobi/above%
  }{
    \path (\x,\y) node[state] (\name) {} node[\pos=.6em,font=\scriptsize] {\dispname};
  }

  \begin{scope}
    \tikzstyle{every node}=[font=\scriptsize]
    \path[->] (actor1)
    edge node[left] {played} (char1)
    edge node[above] {starredIn} (movie1) ;

    \path[->] (char1)
    edge node[above,sloped] {characterIn} (movie1) ;

    \path[->] (actor2)
    edge node[right] {played} (char2)
    edge node[above] {starredIn} (movie2) ;

    \path[->] (char2)
    edge node[above,sloped] {characterIn} (movie2) ;

    \path[<-] (genre)
    edge node[above,sloped] {genre} (movie1)
    edge node[above,sloped] {genre} (movie2) ;
  \end{scope}

\end{tikzpicture}
\caption{Example knowledge graph taken from~\cite{survey}}
\label{fig:kg}
\end{figure}
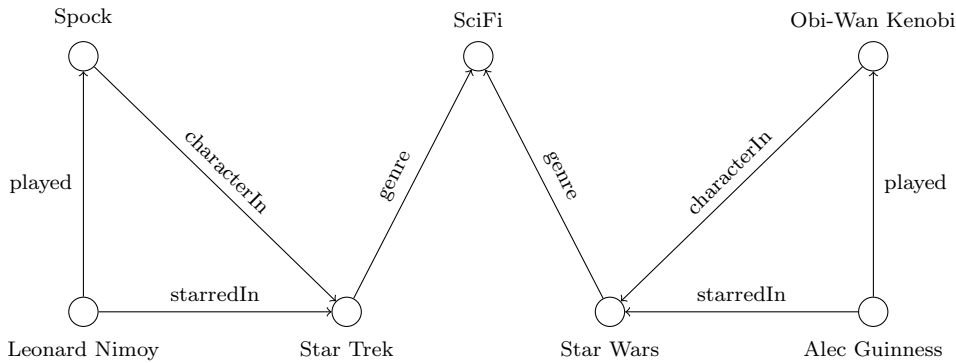
Knowledge graphs, such as YAGO~\cite{yago} and Freebase~\cite{freebase},
have proven useful in many applications such as question answering~\cite{qa}, dialog~\cite{diag}
and recommender~\cite{rec} systems.
A knowledge graph consists of triples $(e_i, e_j, r_k)$
each of which represents that relation $r_k$ holds between subject entity
$e_i$ and object entity $e_j$.
For example, in Figure~\ref{fig:kg},
there exists a triple (Leonard Nimoy, Star Trek, starredIn).
Although a typical knowledge graph may have billions of
triples, it is still far from complete.
Filling in the missing triples is of importance in carrying out
various inference over knowledge graphs.
\emph{Knowledge graph completion}~(KGC) aims to perform this task automatically.

In recent years, \emph{knowledge graph embedding} (KGE) 
has been actively investigated as a promising approach to KGC.
In KGE, entities and relations are embedded
in vector space,
and operations in this space are used 
to define a confidence score (or simply \emph{score}) function $\theta_{i j k}$
that approximates the truth value of a given triple $(e_i,e_j,r_k)$.
Although a variety of original KGE methods~\cite{transe,rescal,ntn,complex,conve}
have been proposed,
Kazemi and Poole~
\cite{simple} and
Lacroix et~al.~
\cite{cano} found
that 
a classical tensor factorization algorithm,
i.e., 
CAN\-DECOMP\slash PARA\-FAC (CP) 
decomposition~\cite{cp},
achieves the state-of-art performance on several benchmark datasets for KGC.

In CP decomposition of a knowledge graph, the confidence score $\theta_{i j k}$ for
a triple $(e_i,e_j,r_k)$ is calculated simply by 
$\mat{a}_{e_i}\transpose(\mat{b}_{e_j}\circ\mat{c}_{r_k})$
where $\mat{a}_{e_i}$, $\mat{b}_{e_j}$, and $\mat{c}_{r_k}$ denote the $D$-dimensional
vectors representing $e_i$, $e_j$, and $r_k$, respectively,
and $\circ$ is the Hadamard~(element-wise) product.
Despite the model's simplicity, it
needs to maintain $(2N_e+N_r)$ $D$-dimensional 32-bit or 64-bit valued
vectors, where $N_e$ and $N_r$ denote the number
of entities and relations, respectively.
Typical knowledge graphs contain an enormous number of entities and relations,
which leads to
significant memory requirements.
In fact, CP with $D=200$ applied to Freebase
will require approximately 66 GB of memory to store parameters.
Large memory consumption can be problematic when
KGC is run on resource-limited devices.
Moreover, the size of existing knowledge graphs is still growing rapidly, and
a method to shrink the embedding vectors is in strong demand.

To address this problem, we present a new CP decomposition
algorithm, which we refer to as B-CP,
to learn compact KGEs.
The basic idea is to introduce
a quantization function into the optimization problem.
This function forces the embedding vectors to be binary,
and optimization is performed with respect to the binarized vectors.
After training, the binarized embeddings can be used in place of the original vectors of floating-point numbers,
which drastically reduces the memory footprint of the resulting model.
In addition, the binary vector representation
contributes to the efficient computation of 
the dot product 
by using bitwise operations.
This fast computation 
allows
the proposed model to significantly
reduce the time
required to compute the confidence scores of triples.

\begin{table}[t]
\centering
\caption{Time complexity, score function, and full expressiveness of KGE models\protect\footnotemark}

\begin{tabular}{lccc}
\toprule
Model              & Time complexity & Score function                                                                     & Fully expressive \\\cmidrule(lr){1-4}
TransE             & $O(D)$          & $- \|\mat{a}_{e_i}+\mat{b}_{r_k}-\mat{a}_{e_j} \|$                                 &                  \\
RESCAL             & $O(D^2)$        & $\mat{a}_{e_i}\transpose\mat{B}_{r_k}\mat{a}_{e_j}$                                & \checkmark       \\ 
DistMult           & $O(D)$          & $\mat{a}_{e_i}\transpose(\mat{a}_{e_j}\circ\mat{b}_{r_k})$                         &                  \\
CP                 & $O(D)$          & $\mat{a}_{e_i}\transpose(\mat{b}_{e_j}\circ\mat{c}_{r_k})$                         & \checkmark       \\\cmidrule(lr){1-4}  
{\bf B-CP}         & $O(D)$          & $-h(\mat{{\bar a}}_{e_i},\text{XNOR}(\mat{{\bar b}}_{e_j}, \mat{{\bar c}}_{r_k}))$ & \checkmark       \\ 
\bottomrule
\end{tabular}
\label{tab:model_exp}
\end{table}
\footnotetext{See Section~\ref{sec:proposed}
for a detailed explanation of the score function and full expressiveness of the propsoed B-CP model.
Though the time complexity of B-CP is $O(D)$, in practice, its score computation is substantially faster than other models
because the Hamming distance function $h(\cdot, \cdot)$ and the $\text{XNOR}(\cdot, \cdot)$ operation can be computed using bitwise operations.}


Note that B-CP only improves the speed and memory footprint 
when predicting missing triples; 
it does not improve the speed and memory footprint 
when training a prediction model.
However,
the reduced memory footprint of the produced model 
enables KGC to be run on many affordable resource-limited devices (e.g., personal computers).
Unlike research-level benchmarks in which one is required to compute the scores of a small set of test triples,
completion of an entire knowledge graph potentially requires computing the scores of an enormous number of missing triples in an inherently sparse knowledge graph,
and
thus, improved memory footprints and reduced score computation time are of practical importance.
B-CP provides exactly these advantages.

From a theoretical perspective, it is important for
a KGE model to have sufficient expressive power to
accurately represent knowledge graphs that contain several relation types~\citep{kgexp}.
Ideally, a KGE model needs to be \emph{fully expressive},
in the sense that,
for any knowledge graph,
there exists an assignment of values to the embeddings of entities and
relations that accurately reconstruct the knowledge graph.
In this paper, we prove the full expressivity of B-CP.
The overall results are summarized in Table~\ref{tab:model_exp}.

Experimental results on several KGC benchmark datasets
showed that, compared to the standard CP decomposition, B-CP
reduced the model size nearly 10- to 20-fold
compared to the standard CP decomposition
without a decrease in the KGC performance.
In addition,
B-CP speeds up score computation considerably by using bitwise operations.



\section{Related Work}
\label{sec:related}
\subsection{KGEs}
Approaches to KGE
can be classified as models
based on bilinear mapping, translation, and neural network-based transformation.

RESCAL~\cite{rescal} is a bilinear-based KGE method whose score function
is formulated as $\theta_{ijk}=\mat{a}_{e_i}^{\rm T}\mat{B}_{r_k}\mat{a}_{e_j}$,
where $\mat{a}_{e_i}, \mat{a}_{e_j} \in \Rset^D$ are vector representations
of entities $e_i$ and $e_j$, respectively,
and matrix $\mat{B}_{r_k} \in \Rset^{D\times D}$ represents a relation $r_k$.
Although RESCAL can output non-symmetric score functions,
each relation matrix $\mat{B}_{r_k}$ holds $D^2$ parameters.
This can be problematic both in terms of overfitting and computational cost.
Several methods that address this problem have been proposed recently.
DistMult~\cite{distmult} restricts the relation matrix to be
diagonal, $\mat{B}_{r_k} = \diag(\mat{b}_{r_k})$.
However, this form of function is necessarily symmetric in $i$ and $j$;
i.e., $\theta_{ijk}=\theta_{jik}$.
To reconcile efficiency and expressiveness,
Trouillon et~al.~(2016)~\cite{complex}
proposed ComplEx, 
which uses the complex-valued representations and a
Hermitian inner product to define the score function,
which, unlike DistMult, can be nonsymmetric in $i$ and $j$.
Hayashi and Shimbo~(2017)~\cite{eq} found that ComplEx is equivalent to another state-of-the-art KGE method, i.e.,
holographic embeddings~(HolE)~\cite{hole}.
ANALOGY~\cite{analogy} is a model that can be considered a hybrid of ComplEx and DistMult.
Manabe et~al.~(2018)~\cite{l1} reduced redundant ComplEx parameters
with L1 regularizers.
Lacroix et~al.~(2018)~\cite{cano} and
Kazemi and Pool~(2018)~\cite{simple} independently showed
that CP decomposition, which Kazemi and Pool refer to as SimplE in the paper 
\cite{simple}, achieves comparable
to that of 
other bilinear methods, such as ComplEx and ANALOGY.
To achieve this level of performance,
they introduced an ``inverse'' triple $(e_j,e_i,r_k^{-1})$ to the training data
for each existing triple $(e_i, e_j, r_k)$,
where $r_k^{-1}$ denotes the inverse relation of $r_k$.

TransE~\cite{transe} is the first KGE model based on vector translation.
It employs the principle ${\mat a}_{e_i}+{\mat b}_{r_k}\approx{\mat a}_{e_j}$
to define a distance-based score function $ \theta_{ijk} = - \|{\mat a}_{e_i}+{\mat b}_{r_k}-{\mat a}_{e_j}\|^{2}$.
TransE was recognized as too limited to
model complex properties (e.g., symmetric/reflexive/one-to-many/many-to-one relations) in knowledge graphs; consequently, 
many extended versions of TransE have been proposed
\cite{transr,stranse,transh}.

Neural-based models, such as Neural Tensor Network~(NTN)~\cite{ntn} and ConvE~\cite{conve},
employ non-linear functions to define a score function;
thus, neural-based models have better expressiveness.
However, compared to bilinear and translation approaches,
neural-based models require more complex operations to compute
interactions between a relation and two entities in vector space.

Note that the binarization technique proposed in this paper can be applied
to KGE models other than CP decomposition, such as those mentioned above.
Our choice of CP as the implementation platform only reflects the fact that
it is one of the strongest baseline KGE methods.

\subsection{Model Compression via Quantization}
Numerous recent publications have investigated methods to train quantized
neural networks to reduce model size
without performance degradation.
Courbariaux, Bengio, and David \cite{bicon} were the first to demonstrate that binarized neural networks
can achieve close to state-of-the-art results 
on datasets, such as MNIST and CIFAR-10~\cite{surveyq}.
Their BinaryConnect method uses the binarization function $Q_1(x)$
to replace floating-point weights of deep neural networks
with binary weights during forward and backward propagation.
Lam~(2018)~\cite{w2b} used the same quantization method as BinaryConnect
to learn compact word embeddings.
To binarize KGEs,
we also apply the quantization method to
the CP decomposition algorithm.
To the best of our knowledge, this technique has not been studied
in the field of tensor factorization.
The primary contribution of this study is that
we introduce a quantization function
to a tensor factorization model.
Note that this study is also the first to investigate the benefits
of quantization for KGC.

\subsection{Boolean Tensor Factorization}
Boolean tensor factorization was formally defined in the paper~\cite{boolean}.
Given a $K$-way boolean tensor
$\mat{\mathcal{X}}\in\{0,1\}^{N_1\times N_2\times\dots\times N_K}$,
the boolean CP decomposition~(boolean CP) factorizes the tensor
to $K$ boolean factor matrices $\mat{A}^{(k)}\in\{0,1\}^{N_k\times D}$
using boolean arithmetic~(i.e., defining $1+1=1$):
$\mat{\mathcal{X}}\approx\bigvee_{d\in[D]}\mat{a}_d^{(1)}\boxtimes\dots\boxtimes\mat{a}_d^{(K)}$
where $\boxtimes$ and $\bigvee$ are logical AND and OR operations, respectively.
Similar to our proposed model,
boolean CP has binary parameters; however, 
its primary purpose is to reconstruct a given tensor accurately with few parameters
rather than tensor completion.
Actually, there have been few practical applications
of boolean tensor factorization.
While boolean CP would be also an interesting research topic for KGC,
its performance cannot be directly evaluated with ranking metrics,
which are the current de~facto standard for evaluating KGE models.




\section{Notation and Preliminaries}
\label{sec:notation}
For the most part, we follow previously established notation and terminology in the paper~\cite{tens}.
The notation and terminology we use for third-order tensors,
by which a knowledge graph is represented
(Section~\ref{sec:kgrep}) are summarized in the following.

Vectors are represented by boldface lowercase letters, e.g., $\mat{a}$.
Matrices are represented by boldface capital letters, e.g., $\mat{A}$.
Third-order tensors are represented by boldface calligraphic
letters, e.g., $\mat{\mathcal{X}}$.
For a natural number $n$,  $ [n] $ denotes the set of natural numbers $\{ 1, 2, \cdots, n \} $.


The $i$th row of a matrix $\mat{A}$ is represented by $\mat{a}_{i:}$,
and
the $j$th column of $\mat{A}$ is represented by 
$\mat{a}_{:j}$,
or simply as $\mat{a}_j$.
The $k$th frontal slice of a third-order tensor
is represented by $\mat{X}_k$.
The symbol $\circ$ represents the Hadamard product
for matrices and vectors,
and $\otimes$ represents the 
outer product.

A third-order tensor $\mat{\mathcal{X}} \in \mathbb{R}^{I_1 \times I_2 \times I_3}$
is rank-one if it can be written as the outer product of three vectors, i.e.,
$
\mat{\mathcal{X}}=\mat{a} \otimes \mat{b} \otimes \mat{c}
$.
This means that each element
$ x_{i_1 i_2 i_3} $
of $\mat{\mathcal{X}}$ is the product of the corresponding
vector elements:
\begin{displaymath}
x_{i_1 i_2 i_3}=a_{i_1}b_{i_2}
c_{i_3} \quad \text{for }i_1 \in [I_1]\text{, } i_2 \in [I_2]\text{, }
i_3 \in [I_3].
\end{displaymath}

The norm of a tensor $\mat{\mathcal{X}} \in \mathbb{R}^{I_1 \times
I_2 \times \cdots \times I_k}$
is the square root of the sum of the squares
of all its elements, i.e.,
\begin{displaymath}
\|\mat{\mathcal{X}}\|=\sqrt{\sum_{ i_1 \in [I_1] } \sum_{ i_2 \in [I_2] } \cdots \sum_{ i_k \in [I_k] }x_{i_1 i_2 \cdots i_k}^2}.
\end{displaymath}
For a matrix (or a second-order tensor), this norm
is Frobenius norm and is denoted $\|\cdot\|_{\text{F}}$.



\section{Tensor Decomposition for Knowledge Graphs}
\label{sec:tensor}
\subsection{Knowledge Graph Representation}
\label{sec:kgrep}
A knowledge graph $\mathcal{G}$ 
is a labeled multigraph 
$(\mathcal{E},\mathcal{R},\mathcal{F})$,
where $\mathcal{E} = \{e_1,\ldots,e_{N_e}\}$ is the set of entities (vertices),
$\mathcal{R} = \{r_1, \ldots ,r_{N_r}\}$
is the set of all relation types (edge labels), and
$\mathcal{F} \subset \mathcal{E} \times \mathcal{E} \times \mathcal{R}$
denotes the observed instances of relations over entities (edges).
The presence of an edge, or a triple, $(e_i,e_j,r_k) \in \mathcal{F}$ represents the fact 
that relation $r_k$ holds between subject entity $e_i$
and object entity $e_j$. 

A knowledge graph can be represented as a 
boolean third order tensor $\mat{\mathcal{X}}
\in \{0,1\}^{N_e \times N_e \times N_r }$
whose elements are given by
\begin{displaymath}
x_{ijk} = 
\begin{cases} 
   1 & \text{if } (e_i,e_j,r_k) \in \mathcal{F}, \\
   0 & \text{otherwise.}
\end{cases}
\end{displaymath}
KGC is concerned 
with incomplete knowledge graphs, i.e.,
$\mathcal{F} \subsetneq \mathcal{F}^*$,
where $\mathcal{F}^* \subset \mathcal{E} \times \mathcal{E} \times \mathcal{R}$
is the set of ground truth facts (and a superset of observed facts $\mathcal{F}$).
KGE has been recognized as a promising approach to predict
the truth value of unobserved triples 
in $( \mathcal{E} \times \mathcal{E} \times \mathcal{R} ) \setminus \mathcal{F}$.
KGE can be generally formulated as a tensor factorization problem
and defines a score function $\theta_{ijk}$ using
the latent vectors of entities and relations.


\subsection{CP Decomposition for KGC}
\begin{figure*}[t]
\centering
\begin{tikzpicture}[scale=0.83]



\pgfmathsetmacro{\cubex}{2}
\pgfmathsetmacro{\cubey}{2}
\pgfmathsetmacro{\cubez}{2}
\draw[black] (0,0,0) -- ++(-\cubex,0,0) -- ++(0,-\cubey,0) -- ++(\cubex,0,0) -- cycle;
\draw[black] (0,0,0) -- ++(0,0,-\cubez) -- ++(0,-\cubey,0) -- ++(0,0,\cubez) -- cycle;
\draw[black] (0,0,0) -- ++(-\cubex,0,0) -- ++(0,0,-\cubez) -- ++(\cubex,0,0) -- cycle;


\node (subject) at (-1,-2,0) [below] {object} ;
\node (object) at (-2.2,-1,0) [rotate=90] {subject} ;
\node (relation) at (-1.85,0.5,0) [rotate=45] {relation} ;

\pgfmathsetmacro{\x}{0}
\pgfmathsetmacro{\y}{-0.2}
\pgfmathsetmacro{\z}{0.3}
\pgfmathsetmacro{\cubex}{1.2 - \x}
\pgfmathsetmacro{\cubexx}{0.8}
\pgfmathsetmacro{\cubey}{1.2 - \y}
\pgfmathsetmacro{\cubez}{1.2 - \z}
\draw[black,dotted] (-1.2,-1.2+\y,-1.2+\z) -- ++(-\cubexx,0,0) ;
\draw[black,dotted] (-1.2+\x,0,-1.2+\z) -- ++(0,-\cubey,0) ;
\draw[black,dotted] (-1.2+\x,-1.2+\y,0) -- ++(0,0,-\cubez) ;
\pgfmathsetmacro{\cubex}{0.25}
\pgfmathsetmacro{\cubey}{0.25}
\pgfmathsetmacro{\cubez}{0.25}
\filldraw[draw=black,fill=kred] (-1.075+\x,-1.075+\y,-1.075+\z) -- ++(-\cubex,0,0) -- ++(0,-\cubey,0) -- ++(\cubex,0,0) -- cycle;
\filldraw[draw=black,fill=kred] (-1.075+\x,-1.075+\y,-1.075+\z) -- ++(0,0,-\cubez) -- ++(0,-\cubey,0) -- ++(0,0,\cubez) -- cycle;
\filldraw[draw=black,fill=kred] (-1.075+\x,-1.075+\y,-1.075+\z) -- ++(-\cubex,0,0) -- ++(0,0,-\cubez) -- ++(\cubex,0,0) -- cycle;

\node (relation) at (-1.4+\x,-1.5+\y,0) [above] {$k$} ;
\node (relation) at (-2,-1.2+\y,-1.2+\z) [above] {$i$} ;
\node (relation) at (-1.0+\x,-0.2,-1.2+\z) [above] {$j$} ;

\node (xijk) at (-0.775,-1.375,-1.075) [above] {$x_{ijk}$} ;

\node (equal) at (1.2,-0.9,0) [above,font=\LARGE] {$=$} ;

\pgfmathsetmacro{\cubex}{0.3}
\pgfmathsetmacro{\cubey}{0.3}
\pgfmathsetmacro{\cubez}{2}
\filldraw [draw=kblue, fill=kblue] (2,-0.85) -- ++(-\cubex,0) -- ++(0,-\cubey) -- ++(\cubex,0) -- cycle;
\pgfmathsetmacro{\cubey}{2}
\draw[black] (2,0) -- ++(-\cubex,0) -- ++(0,-\cubey) -- ++(\cubex,0) -- cycle;

\node (A) at (1.85,-2.5,0) [above] {$\mat{a}_1$} ;
\node (B) at (3.2,-0.3,0) [below] {$\mat{b}_1$} ;
\node (C) at (2.65,1.5,0) [below] {$\mat{c}_1$} ;
\node (i) at (2.25,-0.7,0) [below] {$i$} ;
\node (j) at (3.2,0.7,0) [below] {$j$} ;
\node (k) at (1.85,1.0,0) [below] {$k$} ;

\pgfmathsetmacro{\cubex}{0.3}
\pgfmathsetmacro{\cubey}{0.3}
\pgfmathsetmacro{\cubez}{2}
\filldraw[draw=kyellow, fill=kyellow] (3.35,0) -- ++(-\cubex,0) -- ++(0,-\cubey) -- ++(\cubex,0) -- cycle;

\pgfmathsetmacro{\cubex}{2}
\draw[black] (4.2,0) -- ++(-\cubex,0) -- ++(0,-\cubey) -- ++(\cubex,0) -- cycle;

\pgfmathsetmacro{\cubex}{0.3}
\pgfmathsetmacro{\cubey}{2}
\pgfmathsetmacro{\cubez}{0.3}
\filldraw[draw=kgreen, fill=kgreen] (2,0.2,-0.85) -- ++(-\cubex,0,0) -- ++(0,0,-\cubez) -- ++(\cubex,0,0) -- cycle;
\pgfmathsetmacro{\cubez}{2}
\draw[black] (2,0.2,0) -- ++(-\cubex,0,0) -- ++(0,0,-\cubez) -- ++(\cubex,0,0) -- cycle;

\pgfmathsetmacro{\span}{3.5}
\pgfmathsetmacro{\cubex}{0.3}
\pgfmathsetmacro{\cubey}{0.3}
\pgfmathsetmacro{\cubez}{2}
\filldraw [draw=kblue, fill=kblue] (\span+2,-0.85) -- ++(-\cubex,0) -- ++(0,-\cubey) -- ++(\cubex,0) -- cycle;
\pgfmathsetmacro{\cubey}{2}
\draw[black] (\span + 2,0) -- ++(-\cubex,0) -- ++(0,-\cubey) -- ++(\cubex,0) -- cycle;

\node (A) at (\span+1.85,-2.5,0) [above] {$\mat{a}_2$} ;
\node (B) at (\span+3.2,-0.3,0) [below] {$\mat{b}_2$} ;
\node (C) at (\span+2.65,1.5,0) [below] {$\mat{c}_2$} ;
\node (i) at (\span+2.25,-0.7,0) [below] {$i$} ;
\node (j) at (\span+3.2,0.7,0) [below] {$j$} ;
\node (k) at (\span+1.85,1.0,0) [below] {$k$} ;

\pgfmathsetmacro{\cubey}{0.3}
\pgfmathsetmacro{\cubez}{2}
\pgfmathsetmacro{\cubex}{0.3}
\filldraw[draw=kyellow, fill=kyellow] (\span+3.35,0) -- ++(-\cubex,0) -- ++(0,-\cubey) -- ++(\cubex,0) -- cycle;
\pgfmathsetmacro{\cubex}{2}
\draw[black] (\span+4.2,0) -- ++(-\cubex,0) -- ++(0,-\cubey) -- ++(\cubex,0) -- cycle;

\pgfmathsetmacro{\cubex}{0.3}
\pgfmathsetmacro{\cubey}{2}
\pgfmathsetmacro{\cubez}{0.3}
\filldraw[draw=kgreen, fill=kgreen] (\span+2,0.2,-0.85) -- ++(-\cubex,0,0) -- ++(0,0,-\cubez) -- ++(\cubex,0,0) -- cycle;
\pgfmathsetmacro{\cubez}{2}
\draw[black] (\span+2,0.2,0) -- ++(-\cubex,0,0) -- ++(0,0,-\cubez) -- ++(\cubex,0,0) -- cycle;

\node (plus) at (\span + 1.2,-0.9,0) [above,font=\LARGE] {$+$} ;

\pgfmathsetmacro{\span}{9}
\pgfmathsetmacro{\cubex}{0.3}
\pgfmathsetmacro{\cubez}{2}
\pgfmathsetmacro{\cubey}{0.3}
\filldraw [draw=kblue, fill=kblue] (\span+2,-0.85) -- ++(-\cubex,0) -- ++(0,-\cubey) -- ++(\cubex,0) -- cycle;
\pgfmathsetmacro{\cubey}{2}
\draw[black] (\span + 2,0) -- ++(-\cubex,0) -- ++(0,-\cubey) -- ++(\cubex,0) -- cycle;

\node (A) at (\span+1.85,-2.5,0) [above] {$\mat{a}_D$} ;
\node (B) at (\span+3.2,-0.3,0) [below] {$\mat{b}_D$} ;
\node (C) at (\span+2.65,1.5,0) [below] {$\mat{c}_D$} ;
\node (i) at (\span+2.25,-0.7,0) [below] {$i$} ;
\node (j) at (\span+3.2,0.7,0) [below] {$j$} ;
\node (k) at (\span+1.85,1.0,0) [below] {$k$} ;

\pgfmathsetmacro{\cubey}{0.3}
\pgfmathsetmacro{\cubez}{2}
\pgfmathsetmacro{\cubex}{0.3}
\filldraw[draw=kyellow, fill=kyellow] (\span+3.35,0) -- ++(-\cubex,0) -- ++(0,-\cubey) -- ++(\cubex,0) -- cycle;
\pgfmathsetmacro{\cubex}{2}
\draw[black] (\span+4.2,0) -- ++(-\cubex,0) -- ++(0,-\cubey) -- ++(\cubex,0) -- cycle;

\pgfmathsetmacro{\cubex}{0.3}
\pgfmathsetmacro{\cubey}{2}
\pgfmathsetmacro{\cubez}{0.3}
\filldraw[draw=kgreen, fill=kgreen] (\span+2,0.2,-0.85) -- ++(-\cubex,0,0) -- ++(0,0,-\cubez) -- ++(\cubex,0,0) -- cycle;
\pgfmathsetmacro{\cubez}{2}
\draw[black] (\span+2,0.2,0) -- ++(-\cubex,0,0) -- ++(0,0,-\cubez) -- ++(\cubex,0,0) -- cycle;

%
%
%
\node (plus) at (\span -1,-0.9,0) [above,font=\LARGE] {$+$} ;
\node (plus) at (\span +1.2,-0.9,0) [above,font=\LARGE] {$+$} ;
\node (dot) at (\span +0.1,-0.9,0) [above,font=\LARGE] {$\cdots$} ;

\end{tikzpicture}
\caption{$D$-component CP model for
a third-order tensor $\bm{\mathcal{X}}$.
The element $x_{ijk}$ of $\bm{\mathcal{X}}$ is given by
$\mat{a}_{i:}(\mat{b}_{j:}\circ\mat{c}_{k:})\transpose$.}
\label{fig:cp}
\end{figure*}
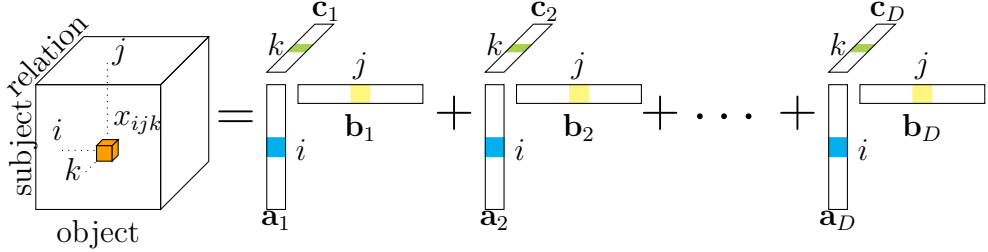
CP decomposition~\cite{cp} factorizes a given tensor
as a linear combination of $D$ rank-one tensors.
For a third-order
tensor $\mat{\mathcal{X}} \in \{0,1\}^{N_e \times N_e \times N_r}$,
its CP decomposition approximates the binary elements directly with real values as the
right-side of the following equation:
\begin{equation}
\mat{\mathcal{X}} \approx \sum_{ d \in [D] } \mat{a}_d \otimes \mat{b}_d \otimes \mat{c}_d \text{~,}
\label{eq:cp}
\end{equation}
where $\mat{a}_d \in \mathbb{R}^{N_e}$, 
$\mat{b}_d \in \mathbb{R}^{N_e}$ and $\mat{c}_d \in \mathbb{R}^{N_r}$.
Figure~\ref{fig:cp} illustrates CP for third-order tensors,
which demonstrates how we can formulate knowledge graphs.
The elements $x_{ijk}$ of $\mat{\mathcal{X}}$ can be written as 
\begin{eqnarray*}
  x_{ijk} \approx (\mat{a}_{i:}\circ \mat{b}_{j:})\mat{c}_{k:}\transpose=\sum_{ d \in [D] }  a_{id}b_{jd}c_{kd} 
  \qquad \text{for $i,j \in [N_e],\ k\in[N_r]$.}
\end{eqnarray*}
A \emph{factor matrix} refers to a matrix composed of vectors
from rank-one components.
We use 
$\mat{A}=[\mat{a}_1 \, \mat{a}_2 \, \cdots \, \mat{a}_D]$ to denote the 
factor matrix,
and denote $\mat{B}$ and $\mat{C}$ in a similar manner.
Note that
$\mat{a}_{i:}$, $\mat{b}_{j:}$ and $\mat{c}_{k:}$ represent the $D$-dimensional embedding vectors of subject $e_i$, object $e_j$,
and relation $r_k$, respectively.

A special case of CP where $\mat{A}=\mat{B}$ is known as DistMult.
DistMult can only model symmetric relations because it does not distinguish between
subject and object entities. However
such a simple model was recently shown to have state-of-the-art results for KGC~\cite{strike}.
Considering these results, we will also evaluate DistMult as a particular case of CP
in our experiments.

\subsection{Logistic Regression for CP Knowledge Graph Embeddings}
\label{sec:logistic}

Following the literature~\cite{logit}, we formulate a logistic regression model 
to solve the CP decomposition problem.
This model considers 
CP decomposition from a probabilistic perspective.
We consider $x_{ijk}$ a random variable and 
compute the maximum a posteriori~(MAP) estimates of
$\mat{A}$, $\mat{B}$, and $\mat{C}$ for the joint distribution as follows:
\begin{displaymath}
p(\mat{\mathcal{X}}|\mat{A},\mat{B},\mat{C})=\prod_{ i \in [N_e] }
\prod_{ j \in [N_e] }\prod_{ k \in [N_r] } p(x_{ijk}|\theta_{ijk}).
\end{displaymath}
We define the score function 
$\theta_{ijk} = \mat{a}_{i:}(\mat{b}_{j:}\circ\mat{c}_{k:})\transpose$.
This score function represents the 
CP decomposition model's confidence that a triple $(e_i,e_j,r_k)$ is a fact;
i.e., that it must be present in the knowledge graph. 
By assuming that $x_{ijk}$ follows a Bernoulli distribution,
$
x_{ijk} \sim \text{Bernoulli} (\sigma(\theta_{ijk}))
$,
the posterior probability is defined as follows:
\begin{displaymath}
p(x_{ijk}|\theta_{ijk}) = 
\begin{cases}
\sigma(\theta_{ijk}) & \quad \text{if } x_{ijk}=1 , \\
1 - \sigma(\theta_{ijk}) & \quad \text{if } x_{ijk}=0 ,
\end{cases}
\end{displaymath}
where $\sigma(x)=1/ ( 1 + \exp(-x) )$ is the sigmoid function.

Furthermore, we minimize the negative log-likelihood of the MAP estimates
such that the general form of the objective function to optimize is

\begin{displaymath}
  E = \sum_{ i \in [N_e] }\sum_{ j \in [N_e] }\sum_{ k \in [N_r] }E_{ijk},
\end{displaymath}
where
\begin{eqnarray*}
\begin{split}
	E_{ijk} &= 
	\underbrace{-x_{ijk}\log{\sigma(\theta_{ijk}}) + (x_{ijk}-1) \log(1- \sigma(\theta_{ijk}))}_{
\let\scriptstyle\textstyle
	\substack{\ell_{ijk}}} \\
	&\qquad \qquad \qquad + \underbrace{ \lambda_A\|\mat{a}_{i:}\|^2 
	+ \lambda_B\|\mat{b}_{j:}\|^2 
	+ \lambda_C\|\mat{c}_{k:}\|^2 }_{
\let\scriptstyle\textstyle
	\substack{\text{L2 regularizer}}}.
\end{split}
\end{eqnarray*}
Here, $\ell_{ijk}$ represents the logistic loss function for
a triple $(e_i,e_j,r_k)$.
While
most knowledge graphs contain only positive examples,
negative examples (false facts) are required to optimize the objective function.
However, if all unknown triples are treated as negative samples,
calculating the loss function requires
a prohibitive amount of time.
To approximately minimize the objective function,
following previous studies,
we used negative sampling in our experiments.

The objective function is minimized with an online learning method based on stochastic gradient descent~(SGD).
For each training example, SGD iteratively updates parameters by
$\mat{a}_{i:} \leftarrow \mat{a}_{i:}-\eta ({\partial E_{ijk}} / {\partial \mat{a}_{i:}}) $,
$\mat{b}_{j:} \leftarrow \mat{b}_{j:}-\eta ({\partial E_{ijk}} / {\partial \mat{b}_{j:}}) $,
and
$\mat{c}_{k:} \leftarrow \mat{c}_{k:}-\eta ({\partial E_{ijk}} / {\partial \mat{c}_{k:}}) $
with a learning rate $\eta$.
The partial gradient of the objective function with respect to
$\mat{a}_{i:}$ is
\begin{align*}
	\frac{\partial E_{ijk}}{\partial \mat{a}_{i:}}
& =
-x_{ijk}
\exp{\left(-\theta_{ijk} \right)}\sigma(\theta_{ijk})
\mat{b}_{j:} \circ \mat{c}_{k:} \\
& \qquad + \left(1-x_{ijk}\right)
\sigma(\theta_{ijk})
\mat{b}_{j:} \circ \mat{c}_{k:}
+2\lambda_A \mat{a}_{i:}.
\end{align*}
Those with respect to $\mat{b}_{j:}$ and
$\mat{c}_{k:}$ can be calculated in the same manner.


\section{Proposed Method}
\label{sec:proposed}
\subsection{Binarized Canonical Polyadic Decomposition}
We propose a B-CP decomposition algorithm
to make CP factor matrices
$\mat{A}$, $\mat{B}$, and $\mat{C}$
binary, i.e., the elements of these matrices are constrained to
only two possible values.

In this algorithm, we formulate the score function
$\theta_{i j k}\binary =\sum_{ d \in [D] } a_{i d}\binary b_{j d}\binary c_{k d}\binary $,
where $a_{i d}\binary =Q_{\Delta}(a_{i d}),\ b_{j d}\binary =Q_{\Delta}(b_{j d}),
\ c_{k d}\binary =Q_{\Delta}(c_{k d})$ are
obtained by binarizing
$a_{i d},\ b_{j d},\ c_{k d}$ through
the following quantization function:
\begin{displaymath}
  Q_{\Delta}(x)=\Delta \sign(x)=
  \begin{cases}
    +\Delta  & \quad \text{if } x \ge 0, \\
    -\Delta & \quad \text{if } x < 0,
  \end{cases}
\end{displaymath}
where
$\Delta$
is a positive constant value.
We extend the binarization function to
vectors in a natural way:
$Q_{\Delta}(\mat{x})$
is a vector whose $i$th element 
is $Q_{\Delta}(x_i)$.

Using the new score function,
we reformulate the loss function defined in Section \ref{sec:logistic}
as follows
\begin{displaymath}
	\ell_{ijk}\binary =
	-x_{ijk}\log{\sigma(\theta_{ijk}\binary )}
	+ (x_{ijk}-1) \log(1- \sigma(\theta_{ijk}\binary )).
\end{displaymath}
To train the binarized CP decomposition model,
we optimize the same objective function $E$
as in Section \ref{sec:logistic},
except we use the binarized loss function given above.
\if
Each element $x_{ijk}$ is approximated as:
 $x_{ijk} \approx H(\sigma(\theta_{ijk}))$,
where step function $H(x)$ is represented as follows:
\begin{equation*}
  H(x) =
\begin{cases}
	1&{\rm if~} x > 0.5,\\
	0&{\rm otherwise}.
\end{cases}
\end{equation*}
We extend the sigmoid and step function to matrices in a natural way:
$\mat{Y} = \sigma(\mat{X}), \mat{Z} = H(\mat{X})$, whose element
$y_{ij}$ and $z_{ij}$ are $\sigma(x_{ij})$ and $H(x_{ij})$, respectively.
 Then,
 the tensor $\mat{\mathcal{X}}$ is approximated as:
 \begin{displaymath}
\mat{\mathcal{X}}
\approx H\left(\sigma \left(\sum_{ d \in [D] } \mat{a}_d \otimes \mat{b}_d \otimes \mat{c}_d\right)\right) \text{.}
 \end{displaymath}
\fi
We also employ the SGD algorithm to minimize the objective function.
Note that the parameters cannot be updated properly
because the gradients of $Q_{\Delta}$ are zero
almost everywhere.
To address this issue, we simply
use an identity matrix $\mat{I}$ as the surrogate for the derivative
of $Q_{\Delta}$:
\begin{displaymath}
\frac{\partial Q_{\Delta}(\mat{x})}{\partial \mat{x}} \approx \mat{I}.
\end{displaymath}
This simple technique enables us to calculate the partial gradient of
the objective function with respect to
$\mat{a}_{i:}$ through the following chain rule:
\begin{displaymath}
	\frac{\partial \ell\binary _{ijk}}{\partial \mat{a}_{i:}}=
\frac{\partial Q_{\Delta}(\mat{a}_{i:})}{\partial \mat{a}_{i:}}
\frac{\partial \ell\binary _{ijk}}{\partial Q_{\Delta}(\mat{a}_{i:})} \approx
\mat{I}
\frac{\partial \ell\binary _{ijk}}{\partial Q_{\Delta}(\mat{a}_{i:})} =
\frac{\partial \ell\binary _{ijk}}{\partial \mat{a}_{i:}\binary }.
\end{displaymath}
This strategy is known as Hinton's straight-through estimator~\cite{hinton}, which 
has been developed in the 
deep neural networks community 
to quantize network components~\cite{bicon,surveyq}.
Using this technique,
we finally obtain the partial gradient as follows:
\begin{align*}
	\frac{\partial E_{ijk}}{\partial \mat{a}_{i:}}
& =
-x_{ijk}
\exp{\left(-\theta_{ijk}\binary  \right)}\sigma(\theta_{ijk}\binary )
\mat{b}_{j:}\binary  \circ \mat{c}_{k:}\binary  \\
& \qquad + \left(1-x_{ijk}\right)
\sigma(\theta_{ijk}\binary )
\mat{b}_{j:}\binary  \circ \mat{c}_{k:}\binary 
+2\lambda_A \mat{a}_{i:}.
\end{align*}
The partial gradients with respect to $\mat{b}_{j:}$ and $\mat{c}_{k:}$
can be computed in a similar manner.



\subsection{Faster Score Computation with Bitwise Operations}
\label{sec:bitwise-computation}

Binary vector representations
result in
faster computation of scores $\theta_{ijk}\binary $,
because
the inner product between binary vectors can
be implemented by bitwise operations:
To compute $\theta_{ijk}\binary $,
we can use an XNOR operation and the Hamming distance function:
\begin{equation}
  \theta_{ijk}\binary
  =
  (\mat{a}_{i:}\binary \circ\mat{b}_{j:}\binary ) \mat{c}_{k:}\transposeb
  =
  \Delta^{3} \{ D-2 BitC \}
  \label{eq:b-cp-quick-computation}
\end{equation}
where $BitC=h(\overline{\mat{a}}_{i:}\binary , \text{XNOR}(\overline{\mat{b}}_{j:}\binary ,\overline{\mat{c}}_{k:}\binary ))$.
Here $\overline{\mat{x}}\binary $ denotes the boolean vector
whose $i$th element $\overline{x}_i\binary $ is set to one if $x_i\binary =\Delta$;
otherwise zero.
$\text{XNOR}$ represents
the logical complement of the exclusive OR operation,
and $h(\cdot,\cdot)$ denotes the Hamming distance function.
Note that,
as shown in Table~\ref{tab:model_exp},
when we are interested in the ranking of triples by 
Eq.~\eqref{eq:b-cp-quick-computation},
computing $BitC$ for each triple is sufficient as $D$ and $\Delta$
are constant over all triples.


\subsection{Full Expressiveness of B-CP}
\label{sec:full-expressiveness}

It is known that the CP model \eqref{eq:cp} 
is \emph{fully expressive}~\cite{simple}, i.e.,
given any knowledge graph,
there exists an assignment of values to
the embeddings of the entities and relations that accurately reconstruct it.
To be precise, there exists a natural number $D$ and a set of matrices $\mat{A} , \mat{B} \in \Rset^{N_e \times D}$ and $\mat{C} \in \Rset^{N_r \times D}$ such that
strict equality holds in Eq.~\eqref{eq:cp}, i.e.,
\begin{displaymath}
  \mat{\mathcal{X}} = \sum_{ d \in [D] } \mat{a}_d \otimes \mat{b}_d \otimes \mat{c}_d .
\end{displaymath}



It can be shown that B-CP is also fully expressive in the following sense.


\begin{thm}
  \label{thm:expressiveness}
  For an arbitrary boolean tensor
  $\mat{\mathcal{X}} \in \{0, 1\}^{N_e \times N_e \times N_r}$,
  there exists a B-CP decomposition with binary factor matrices
  $ \mat{A}\binary  , \mat{B}\binary  \in \{+\Delta, -\Delta\}^{N_e \times D} $  and $\mat{C}\binary \in \{+\Delta, -\Delta\}^{N_r \times D} $
  for some $D$ and $\Delta$,
  such that
  \begin{equation}
    \mat{\mathcal{X}} = \sum_{ d \in [D] } \mat{a}\binary _d \otimes \mat{b}\binary _d \otimes \mat{c}\binary _d .
    \label{eq:bcp-model}
  \end{equation}

\end{thm}

\begin{proof}
  See \ref{sec:proof}.
\end{proof}



\section{Experiments}
\label{sec:exp}
\subsection{Experiments on Benchmark Datasets}
We evaluated the performance of the proposed approach in a standard
KGC task.
\subsubsection{Datasets and Evaluation Protocol}
\begin{table}[t]
  \centering
  \mytablefont
  \caption{Benchmark datasets for KGC.}
  \label{tab:dataset}
  \scalebox{1.5}{
  \begin{tabular}{lrrrr}
  \toprule
                             & WN18 & FB15k & WN18RR & FB15k-237 \\\cmidrule(lr){1-5}
             $N_e$ & 40,943 & 14,951 & 40,559 & 14,505 \\
             $N_r$ & 18 & 1,345 & 11 & 237 \\
             \# training triples & 141,442 & 483,142 & 86,835 & 272,115 \\
             \# validation triples & 5,000 & 50,000 & 3,034 & 17,535 \\
             \# test triples  & 5,000 & 59,071 & 3,134 & 20,466 \\
  \bottomrule
  \end{tabular} 
}
\end{table}



We used four standard datasets,
WN18, FB15k~\cite{transe}, WN18RR, and FB15k-237~\cite{conve}.
Table~\ref{tab:dataset} shows the data statistics\footnote{
Following \cite{simple,cano},
for each triple $(e_i,e_j,r_k)$ observed in the training dataset,
we added its inverse triple $(e_j,e_i,r_k^{-1})$
also in the training set.}.

We followed the standard evaluation procedure
to evaluate the KGC performance:
Given a test triple $(e_i,e_j,r_k)$, we
corrupted it by replacing $e_i$ or $e_j$
with every entity $e_\ell$ in $\mathcal{E}$
and calculated the score $\theta_{i,\ell,k}$ or $\theta_{\ell,j,k}$.
We then ranked all these triples and the original non-corrupted triple by
their scores in descending order.
To measure the quality of the ranking,
we used the mean reciprocal rank (MRR) and Hits at $N$~(Hits@$N$).
We here report only results in the filtered setting~\cite{transe},
which provides a more reliable performance metric in the presence of multiple correct triples.

\subsubsection{Experimental Setup}
\label{sec:expcond}
To train DistMult/CP models, we selected the hyperparameters via a
grid search such that the filtered MRR is maximized on the validation set.
For the standard CP model,
the grid search was performed over
all combinations of $\lambda_A,\lambda_B,\lambda_C\in \{ 0, 0.0001 \}$,
learning rate $\eta\in\{0.025, 0.05\}$,
and embedding dimension $D\in\{\allowbreak 200, \allowbreak 400\}$.
For our binarized CP (B-DistMult/B-CP) models,
all combinations of $\lambda_A,\lambda_B,\lambda_C\in \{ 0, 0.0001 \}$,
$\eta\in\{0.025, 0.05\}$, $\Delta\in\{0.3, 0.5\}$
and $D\in\{200,400,800\}$ were tried.
The initial values of the representation vector components
were randomly sampled from the uniform distribution
$U[- {\sqrt{6}} / {\sqrt{2D}}, {+ \sqrt{6}} / {\sqrt{2D}}]$~\cite{init}.
The maximum number of training epochs was set to 1,000.
For SGD training, negative samples were generated on the basis of
the local closed-world assumption~\cite{survey}.
The number of negative samples
generated per positive sample was five for WN18/WN18RR
and ten for FB15k/FB15k-237.

We implemented our CP decomposition 
systems in C++ and conducted all experiments on a 64-bit
16-Core AMD Ryzen Threadripper 1950x with 3.4 GHz CPUs.
The program code was compiled using GCC 7.3 with the -O3 option.

\subsubsection{Main Results}
      \begin{table*}[tb]
  \mytablefont
  \caption{KGC results on WN18 and FB15k: Filtered MRR and Hits@$\{1,3,10\}$ (\%).
      *, ** and *** indicate results transcribed from~\cite{complex},~\cite{conve} and~\cite{simple}, respectively.}
  \label{tab:results_old}
  \centering
  \scalebox{1.3}{
    \begin{tabular}{lcccccccc}
      \toprule
      & \multicolumn{4}{c}{WN18}   & \multicolumn{4}{c}{FB15k} \\
      \cmidrule(lr){2-5}\cmidrule(lr){6-9}
      & \multirow{2}{*}{MRR} & \multicolumn{3}{c}{Hits@} & \multirow{2}{*}{MRR} & \multicolumn{3}{c}{Hits@} \\
      \cmidrule(lr){3-5}\cmidrule(lr){7-9}
      Models & & 1 & 3 & 10 & & 1 & 3 & 10  \\\cmidrule(lr){1-9}
      TransE*  & 45.4 & 8.9  & 82.3 & 93.4      & 38.0 & 23.1 & 47.2 & 64.1  \\
      DistMult*& 82.2 & 72.8 & 91.4 & 93.6      & 65.4 & 54.6 & 73.3 & 82.4  \\
      HolE*    & 93.8 & 93.0 & 94.5 & 94.9      & 52.4 & 40.2 & 61.3 & 73.9  \\
      ComplEx* & 94.1 & 93.6 & 94.5 & 94.7      & 69.2 & 59.9 & 75.9 & 84.0  \\
      ANALOGY**& 94.2 & 93.9 & 94.4 & 94.7       & 72.5 & 64.6 & 78.5 & 85.4  \\
      CP***    & 94.2 & 93.9 & 94.4 & 94.7       & 72.7 & {\bf 66.0} & 77.3 & 83.9 \\
      ConvE**  & 94.3 & 93.5 & 94.6 & {\bf 95.6} & 65.7 & 55.8 & 72.3 & 83.1  \\\cmidrule(lr){1-9}
      DistMult &  82.4 & 73.1 & 91.8 & 94.0 & 65.3 & 54.2 & 73.0 & 82.1 \\
      CP & 94.2 & 93.9 & 94.5 & 94.7 & 72.0 & 65.9 & 76.8 & 82.9 \\
      {\bf B-DistMult} & 84.1 & 76.1 & 91.5 & 94.4 & 67.2 & 55.8 & 76.0 & 85.4 \\
      {\bf B-CP} ($D=200$) & 90.1 & 88.1 & 91.8 & 93.3 & 69.5 & 61.1 & 76.0 & 83.5 \\
      {\bf B-CP} & {\bf 94.5} & {\bf 94.1} & {\bf 94.8} & {\bf 95.6} & {\bf 73.3} & {\bf 66.0} & {\bf 79.3} & {\bf 87.0} \\\bottomrule
    \end{tabular}
  }
\end{table*}



\begin{table*}[tb]
  \mytablefont
  \caption{KGC results on WN18RR and FB15k-237: Filtered MRR and Hits@$\{1,3,10\}$ (\%).
      * indicates results transcribed from~\cite{conve}.}
  \label{tab:results_new}
  \centering
  \scalebox{1.3}{
    \begin{tabular}{lcccccccc}
      \toprule
      & \multicolumn{4}{c}{WN18RR}   & \multicolumn{4}{c}{FB15k-237} \\
      \cmidrule(lr){2-5}\cmidrule(lr){6-9}
      & \multirow{2}{*}{MRR} & \multicolumn{3}{c}{Hits@} & \multirow{2}{*}{MRR} & \multicolumn{3}{c}{Hits@} \\
      \cmidrule(lr){3-5}\cmidrule(lr){7-9}
      Models    & & 1    & 3    & 10   &      & 1    & 3    & 10  \\\cmidrule(lr){1-9}
      DistMult* & 43.0 & 39.0 & 44.0 & 49.0 & 24.1 & 15.5 & 26.3 & 41.9 \\
      ComplEx*  & 44.0 & 41.0 & 46.0 & 51.0 &24.7 & 15.8 & 27.5 & 42.8 \\
      R-GCN*    & --   & --   & --   & --   & 24.8 &15.3 & 25.8 & 41.7 \\
      ConvE*    & 43.0 & 40.0 & 44.0 & 52.0 & {\bf 32.5} & {\bf 23.7} & {\bf 35.6} & {\bf 50.1} \\\cmidrule(lr){1-9}
      DistMult & 43.0 & 40.0 & 44.0 & 49.0 & 24.0 & 15.3 & 26.0 & 41.8 \\
      CP & 44.0 & 42.0 & 46.0 & 51.0 & 29.0 & 19.8 & 32.2 & 47.9 \\
      {\bf B-DistMult} & 43.0 & 40.0 & 44.0 & 49.0 & 24.3 & 15.6 & 26.7 & 42.1 \\
      {\bf B-CP} ($D=200$) & 45.0 & 43.0 & 46.0 & 50.0 & 27.8 & 19.4 & 30.4 & 44.6 \\
      {\bf B-CP} & {\bf 46.0} & {\bf 44.0} & {\bf 47.0} & {\bf 52.0} & 29.5 & 21.0 & 32.4 & 48.3 \\\bottomrule
    \end{tabular}
  }
\end{table*}



\begin{figure}[t]
\centering
\begin{tabular}{c}
\begin{minipage}{0.5\hsize}
\centering
\includegraphics[width=0.95\linewidth]{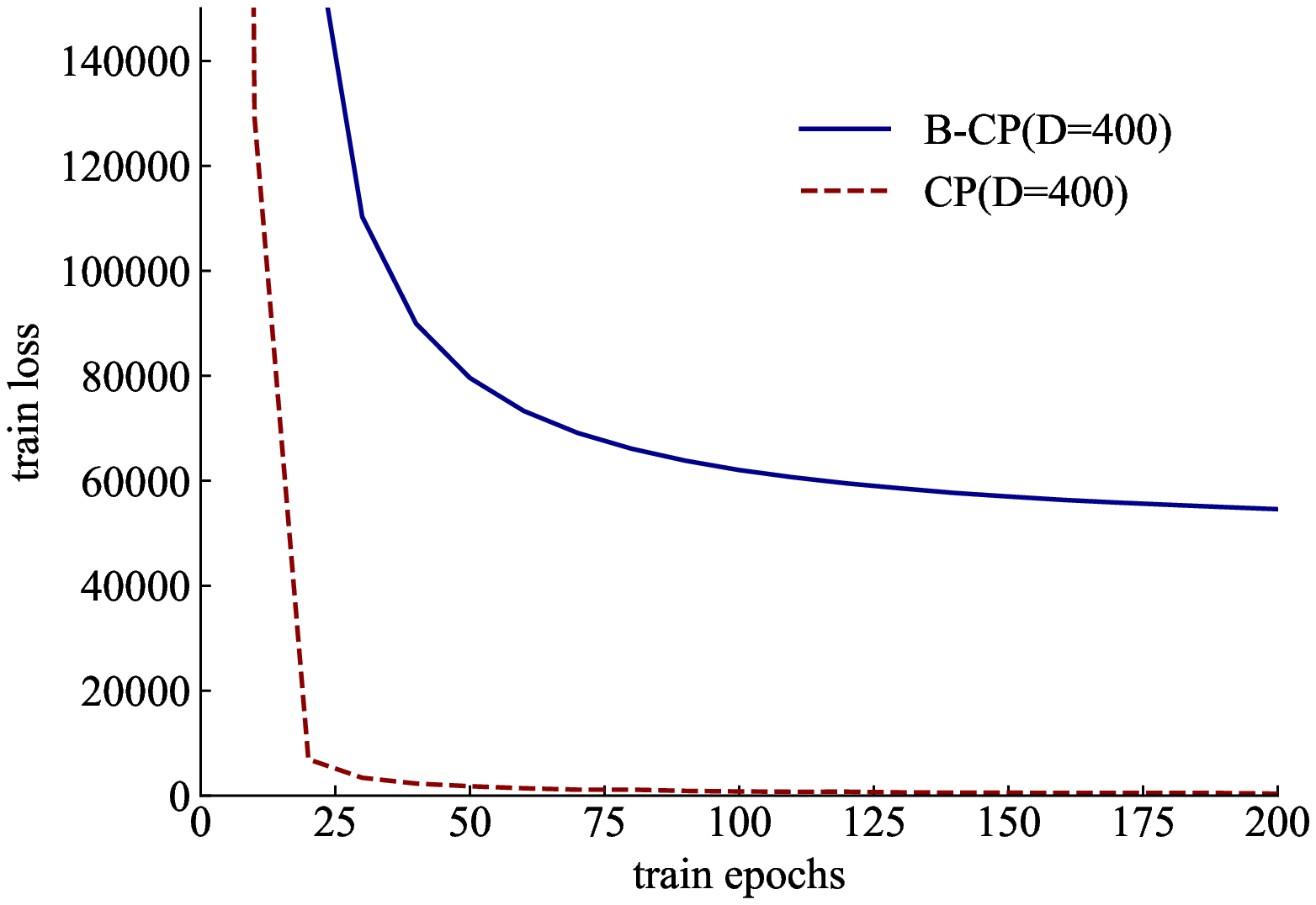}
\end{minipage}
\begin{minipage}{0.5\hsize}
\centering
\includegraphics[width=0.95\linewidth]{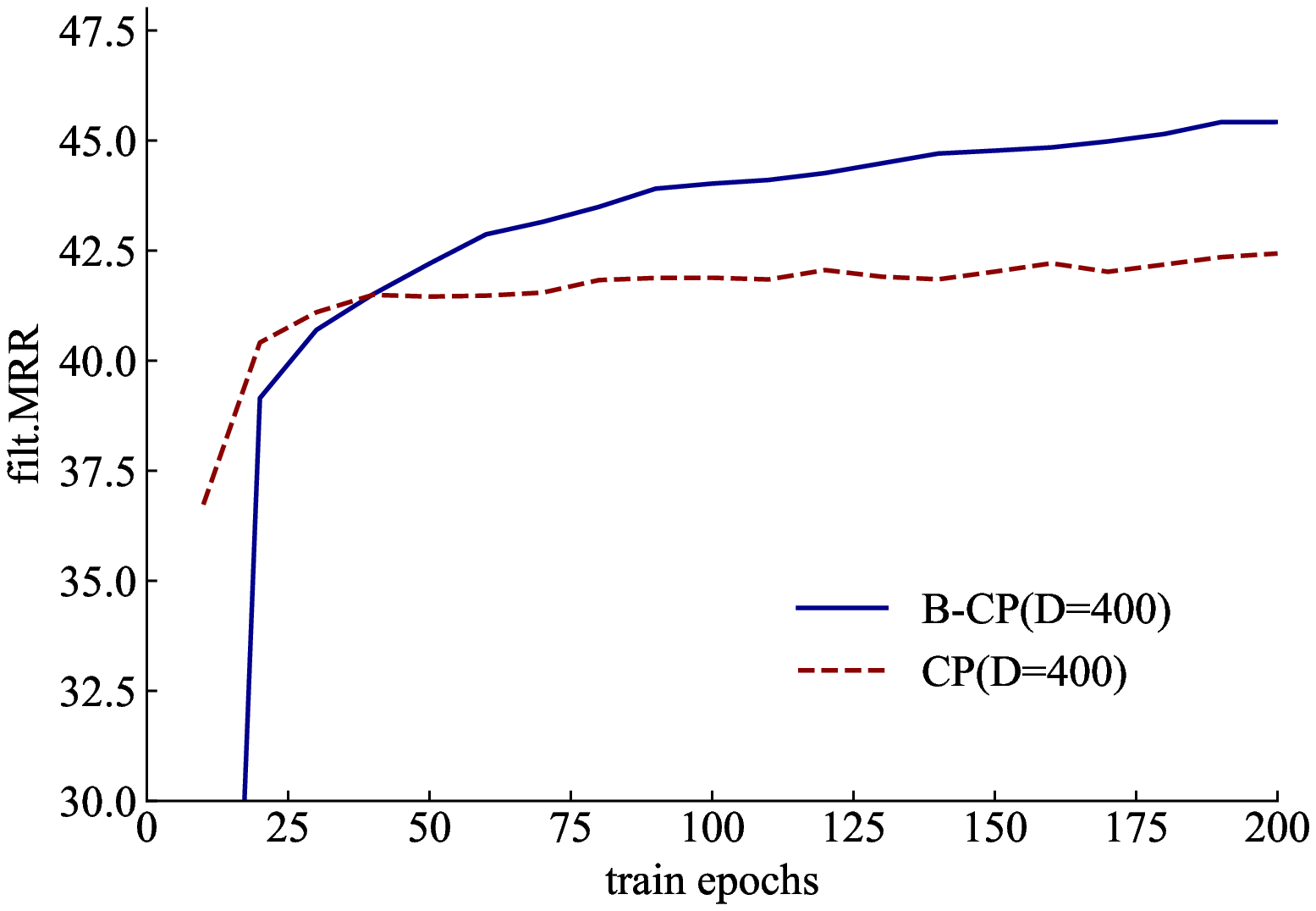}
\end{minipage}
\end{tabular}
\caption{
Training loss and filtered MRR vs. epochs
trained on WN18RR.}
\label{fig:loss-acc}
\end{figure}



We compared standard DistMult/CP and 
B-DistMult/B-CP models
with other state-of-the-art KGE models.
The best vector dimensions were 200 in DistMult/CP for all datasets,
while the best resulting vector dimensions of B-DistMult/B-CP for WN18/WN18RR/FB15k-237
were 400 and those for FB15k were 800.
Table~\ref{tab:results_old} shows
the results on WN18 and FB15k,
and Table~\ref{tab:results_new} shows the results
on WN18RR and FB15k-237.
For most evaluation metrics,
the proposed B-CP model outperformed or was competitive
with the best baseline. However,
with a small vector dimension ($D=200$),
B-CP's performance tended to degrade.

Figure~\ref{fig:loss-acc} shows
training loss and accuracy versus training epochs
for CP~($D=400$) and B-CP~($D=400$) on WN18RR.
The results indicate that CP is prone to overfitting as the 
training epochs increase. 
In contrast, B-CP appears less susceptible to
overfitting than CP.

\subsubsection{KGC Performance vs. Model Size}
\begin{table*}[tb]
  \mytablefont
  \caption{Results on WN18RR and FB15k-237 with varying embedding dimensions. Model size denotes the number of bits required to store the model in memory.}
  \label{tab:results_comp}
  \centering
  \scalebox{1.2}{
    \begin{tabular}{lrcc}
       \toprule
       \multirow{2}{*}{Model} & \multirow{2}{*}{Model Size} & \multicolumn{2}{c}{MRR} \\
       & & WN18RR & FB15k-237 \\\cmidrule(lr){1-4}
       CP ($D=15$)   & $480(2N_e+N_r)$ & 40.0 & 22.0 \\
       CP ($D=50$)   & $1,600(2N_e+N_r)$ & 43.0 & 24.8 \\
       CP ($D=200$)  & $6,400(2N_e+N_r)$ & 44.0 & 29.0 \\
       VQ-CP ($D=200$) & $200(2N_e+N_r)$ & 36.0 & 8.7 \\
       VQ-CP ($D=400$)  & $400(2N_e+N_r)$ & 36.0 & 8.3 \\\cmidrule(lr){1-4}
       {\bf B-CP} ($D=100$) & $100(2N_e+N_r)$ & 38.0 & 23.2 \\
       {\bf B-CP} ($D=200$) & $200(2N_e+N_r)$ & 45.0 & 27.8 \\
       {\bf B-CP} ($D=300$) & $300(2N_e+N_r)$ & 45.0 & 29.0 \\
       {\bf B-CP} ($D=400$) & $400(2N_e+N_r)$ & {\bf 46.0} & 29.2 \\
       {\bf B-CP} ($D=800$) & $800(2N_e+N_r)$ & {\bf 46.0} & {\bf 29.5} \\\bottomrule
    \end{tabular}
  }
\end{table*}



We also investigated the extent to which the proposed B-CP method
can reduce model size and maintain KGC performance.
For a fair evaluation, we also examined a naive vector quantization
method (VQ)~\cite{xnor} that can reduce the model size.
Given a real valued matrix ${\mat X}\in\Rset^{D_1\times D_2}$,
the VQ method solves the following optimization problem:
\begin{align*}
{\hat {\mat X}\binary },{\hat \alpha}=\argmin_{{\mat X}\binary ,\alpha}
{\|{\mat X}-\alpha{\mat X}\binary \|_F^2}
\end{align*}
where ${\mat X}\binary \in\{+1,-1\}^{D_1\times D_2}$ is a binary matrix
and $\alpha$ is a positive real value.
The optimal solutions ${\hat {\mat X}\binary }$ and ${\hat \alpha}$
are given by $Q_1({\mat X})$ and $\frac{1}{D_1\times D_2} \| {\mat X} \|_{1}$,
respectively,
where $\|\cdot \|_1$ denotes $l_1$-norm,
and
$Q_1({\mat X})$ is a sign function whose behavior in each element $x$
of ${\mat X}$ is as per the sign function $Q_1(x)$.
After obtaining factor matrices ${\mat A}$, ${\mat B},$ and ${\mat C}$
via CP decomposition,
we solved the above optimization problem independently for each matrix.
We refer to this method as VQ-CP.

Table~\ref{tab:results_comp} shows the results
when the dimension size of the embeddings was varied.
While CP requires $64\times D$ and $32\times D$ bits per
entity and relation, respectively,
both B-CP and VQ-CP only require one thirty-second of these bit values.
Obviously, the task performance dropped significantly
after vector quantization~(VQ-CP).
The CP performance also degraded when 
the vector dimension was reduced from 200 to 15 or 50.
While simply reducing the number of dimensions degraded
accuracy,
B-CP successfully reduced the
model size nearly 10- to 20-fold compared to CP and other KGE models
without performance degradation.

\subsubsection{Computation Time}
\begin{figure}[t]
\centering
\includegraphics[width=0.8\linewidth]{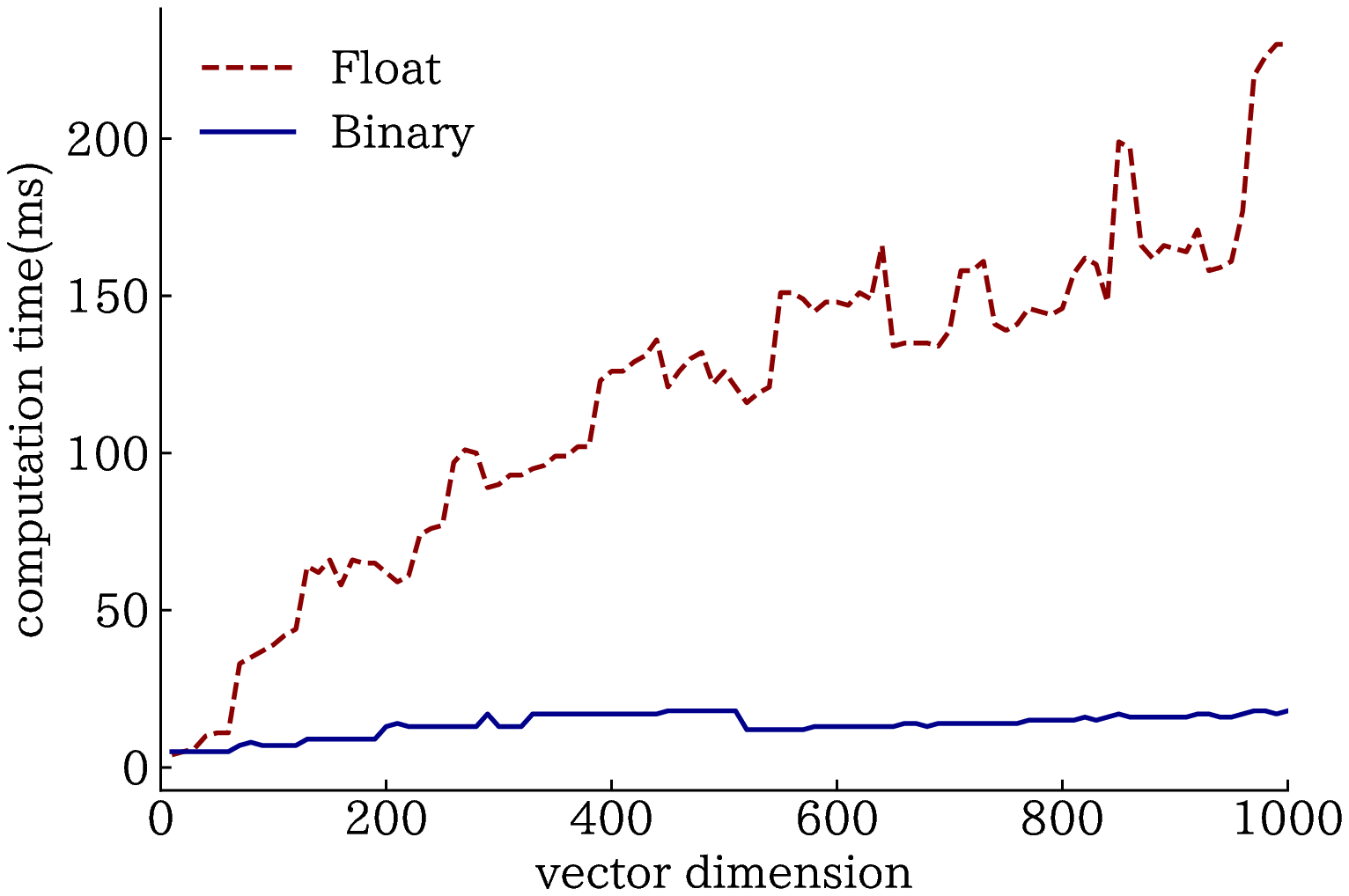}
\caption{CPU run time per 100,000-times score computations with single CPU thread.}
\label{fig:time}
\end{figure}



As described in Section~\ref{sec:proposed},
the B-CP model can
accelerate the computation of confidence scores by using
the XNOR operation and Hamming distance function. 
To compare the score computation speed between CP~(Float) and
B-CP~(Binary),
we calculated the confidence scores 100,000 times for both CP and B-CP while
varying the vector size $D$ from 10 to 1,000 at increments of ten.
Figure~\ref{fig:time} clearly shows that
bitwise operations increase computation speed significantly compared to
standard multiply-accumulate operations.

\subsection{Evaluation on Large-scale Knowledge Graphs}
\begin{table}[t]
  \centering
  \mytablefont
  \caption{Freebase-music and YAGO statistics.}
  \label{tab:ldataset}
  \scalebox{1.3}{
  \begin{tabular}{lrr}
  \toprule
                             & Freebase-music & YAGO \\\cmidrule(lr){1-3}
             $N_e$ & 3,004,505 & 3,983,941\\
             $N_r$ & 131 & 75 \\
             \# training triples & 14,786,254 & 9,944,560\\
             \# validation triples & 3,696,564& 2,486,140\\
             \# test triples  & 3,696,564 & 2,486,140\\
  \bottomrule
  \end{tabular} 
}
\end{table}



\begin{figure}[t]
\centering
\begin{tabular}{c}
\begin{minipage}{0.5\hsize}
\centering
\includegraphics[width=0.85\linewidth]{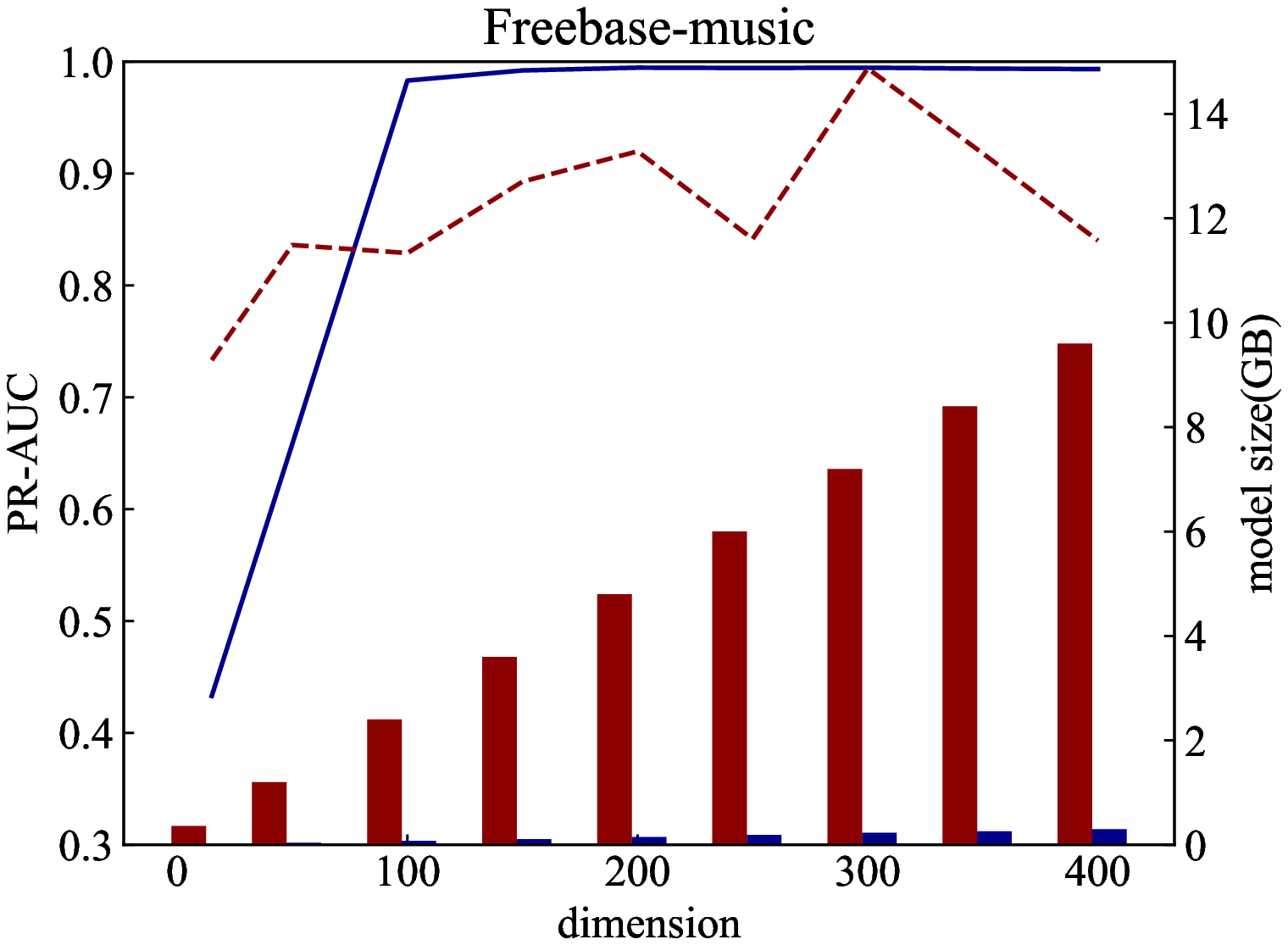}
\end{minipage}
\begin{minipage}{0.5\hsize}
\centering
\includegraphics[width=1.03\linewidth]{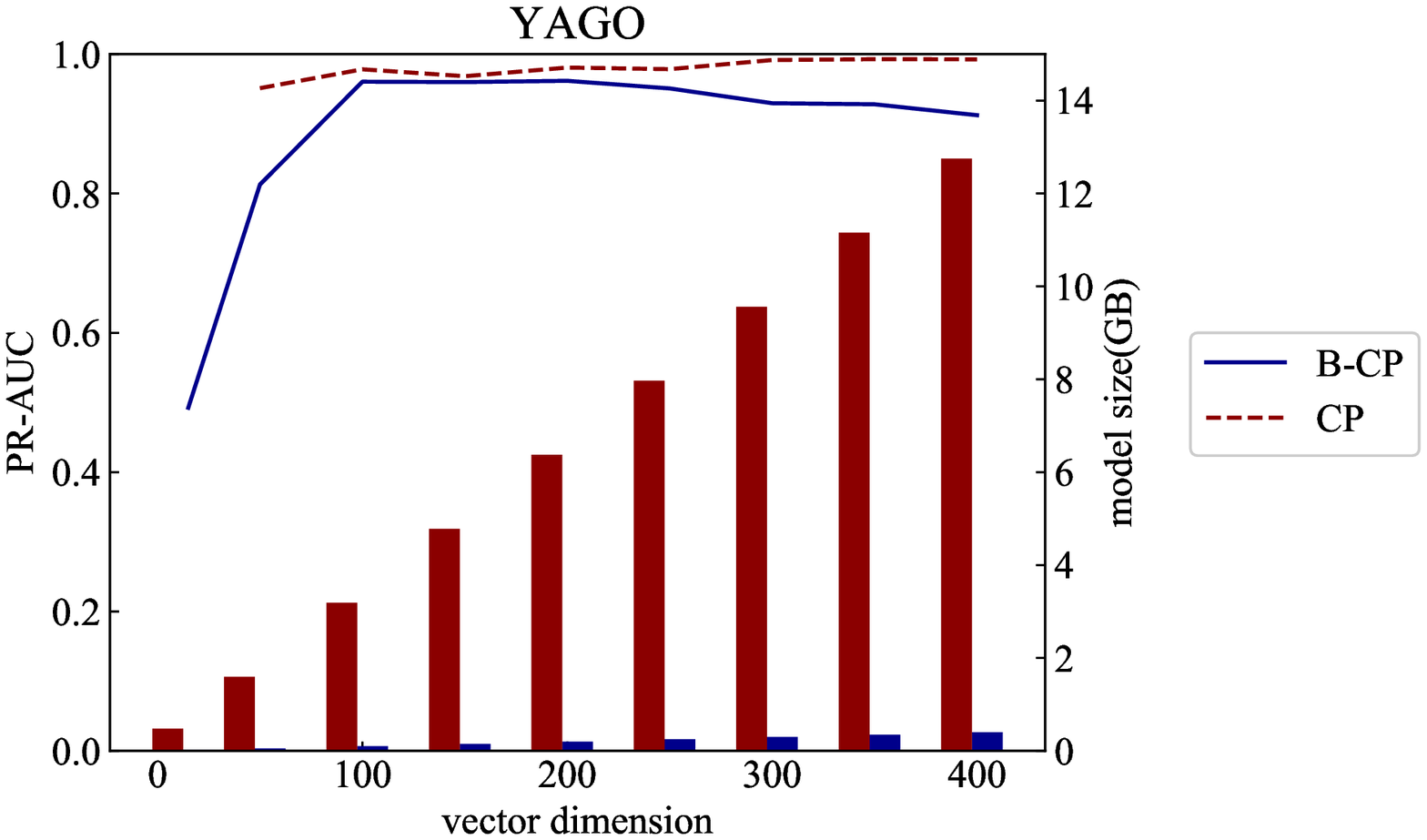}
\end{minipage}
\end{tabular}

\caption{
Results of CP and B-CP models on Freebase-music and YAGO datasets.
The line and bar graphs represent PR-AUC and model size, respectively.
CP model results~($D=15$) on YAGO were excluded
because minimization of the loss function did not converge at all.
}
\label{fig:prauc}
\end{figure}




To verify the effectiveness of B-CP over larger datasets,
we also conducted experiments on the Freebase-music~\cite{freebase_music}
and YAGO~\cite{yago}\footnote{Version3.1 from \url{http://yago-knowledge.org}} datasets.
To reduce noises in Freebase-music, 
we removed the triples whose relations and entities occur less than ten times,
and in YAGO we used only fact triples and excluded other
extra information, such as taxonomies and types of entities.
In both datasets, 
we split the triples into 80/10/10 fractions for training, validation, and testing.
Furthermore, we randomly generated the same number of triples as test~(validation) triples
that were not in the knowledge graph, and added generated triples to the test~(validation) triples 
as negative samples.
Table \ref{tab:ldataset} shows the data statistics.

We tried all combinations of $\lambda_A,\lambda_B,\lambda_C\in\{0.0001, 0\}$,
learning rate $\eta\in\{0.005, 0.0075, 0.01, 0.025, 0.05\}$,
and embedding dimension $D\in\{15, 50, \linebreak 100, 150, 200, 250,300, 350, 400\}$
during the grid search.
We evaluated the results using the area under the 
precision-recall curve~(PR-AUC).

The results are shown in 
Figure \ref{fig:prauc}.
As expected, on each dataset, B-CP successfully reduced the model size while achieving performance equal to or better
than CP.
These results show that B-CP is robust to data size.

\subsection{Link-Based Entity Clustering}
\begin{figure}[t]
\centering
\includegraphics[width=\linewidth]{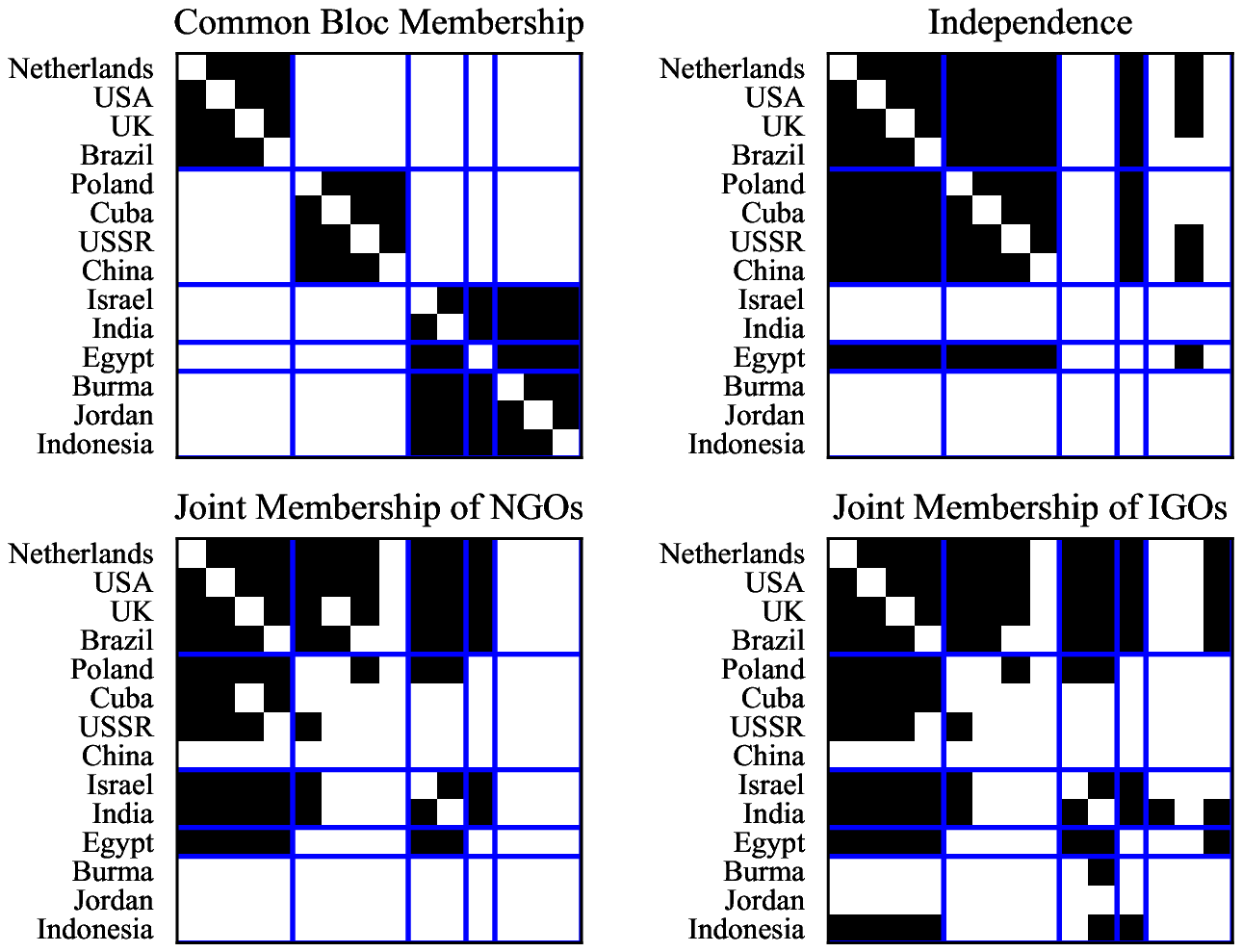}
\caption{Country clusters in the Nations dataset. Each black entry indicates an existing 
relation between two countries.}
\label{fig:clustering}
\end{figure}



Clustering is a useful technique to assess natural groupings of various data items,
including entities in relational databases.
Such cluster information assists knowledge engineers
in the automatic construction of taxonomies from instance data~\cite{factyago}.
In this section,
we demonstrate the utility of B-CP in link-based clustering of the Nations dataset~\cite{kemp}.

The Nations dataset contains 2,024 triples composed of 14 countries and 56 relations.
Experiments were conducted under the same 
hyperparameters that achieved the best results on the 
WN18 dataset. 
We applied hierarchical clustering 
to a factor matrix $\mat{A}\binary $ and divided entities into five clusters,
using the single linkage method with Euclidean distance.
Euclidean distance between binary vectors $\mat{p}\binary $ and
$\mat{q}\binary $ can be computed as follows using Hamming distance,
\begin{equation*}
  d(\mat{p}\binary , \mat{q}\binary )=\sqrt{\sum_{i=1}^N{(p_i\binary -q_i\binary )^2}}
  = 2\sqrt{ h(\overline{\mat{p}}\binary , \overline{\mat{q}}\binary )}
\end{equation*}
which accelerates the computation of clustering.

We show the clustering results in Figure~\ref{fig:clustering}.
The countries are partitioned into 
one group from the western bloc, one
group from the communist bloc,
and three groups for the neutral bloc.
The four relations in Figure~\ref{fig:clustering}
show that this is a reasonable partitioning of Nations dataset,
which indicate that B-CP can accurately
represent the semantic relationships between entities
in binary vector space.

\section{Conclusion}
\label{sec:conclusion}
In this paper,
we have shown that it is possible to obtain binary vectors of entities and relations
in knowledge graphs
that take $10$--$20$ times less storage/memory than
the original representations with floating-point numbers.
In addition, with bitwise operations,
the time required to compute scores was reduced considerably.
Tensor factorization occurs in many machine learning applications,
such as item recommendation~\cite{rec} and web link analysis~\cite{web}.
Applying the proposed B-CP algorithm to the analysis of
other relational datasets
is an interesting avenue for future work.






\appendix


\setcounter{definition}{0}
\renewcommand{\thedefinition}{\Alph{section}.\arabic{definition}}

\setcounter{lemma}{0}
\renewcommand{\thelemma}{\Alph{section}.\arabic{lemma}}

\setcounter{corollary}{0}
\renewcommand{\thecorollary}{\Alph{section}.\arabic{corollary}}

\setcounter{thm}{0}
\renewcommand{\thethm}{\Alph{section}.\arabic{thm}}

\setcounter{figure}{0}

\section{Proof of Theorem~\protect\ref{thm:expressiveness}}
\label{sec:proof}

In this appendix, we prove the full expressiveness of B-CP,
which was stated as Theorem~\ref{thm:expressiveness} without proof in Section~\ref{sec:full-expressiveness}.
To this end,
for a given knowledge graph (or more precisely, its boolean tensor $\mat{\mathcal{X}}$ representing the truth values),
we define a specific B-CP model, denoted by $ \text{BCP}_\Delta^* (\mat{\mathcal{X}}) $,
and show that it indeed faithfully represents $\mat{\mathcal{X}}$ by Eq.~\eqref{eq:bcp-model}
for a certain $\Delta$.

\begin{definition}
  \label{def:bcp-faithful-encoding}
  Let
  $\mat{\mathcal{X}} \in \{0, 1\}^{N_e \times N_e \times N_r}$
  be an arbitrary boolean tensor.
  Let
  $ \text{BCP}_\Delta^* (\mat{\mathcal{X}}) = ( \mat{A}\binary , \mat{B}\binary  , \mat{C}\binary  ) $,
  where
  binary matrices
  $\mat{A}\binary , \mat{B}\binary  \in \{+\Delta, -\Delta\}^{N_e\times 8N_eN_r}$,
  and
  $\mat{C}\binary  \in \{+\Delta, -\Delta\}^{N_r\times 8N_eN_r}$
  are defined as follows.
  \begin{itemize}
  \item
    All three matrices are block matrices of
    $4$-dimensional binary row vectors,
    each of which is either one of 
    \begin{align}
      \mat{p} & = [+\Delta , +\Delta, -\Delta, -\Delta] , \label{eq:p} \\
      \mat{q} & = [+\Delta , -\Delta, +\Delta, -\Delta] , \label{eq:q} \\
      \mat{r} & = [+\Delta , +\Delta, +\Delta, +\Delta] , \label{eq:r}
    \end{align}
    or $-\mat{r}$. 
    In other words,
    $\mat{A}\binary  = [ \mat{a}\binary _{mn} ]_{m \in [N_e] , n \in [2 N_e N_r]}$,
    $\mat{B}\binary  = [ \mat{b}\binary _{mn} ]_{m \in [N_e] , n \in [2 N_e N_r]}$,
    and
    $\mat{C}\binary  = [ \mat{c}\binary _{mn} ]_{m \in [N_r] , n \in [2 N_e N_r]}$,
    where
    $ \mat{a}\binary _{mn} , \mat{b}\binary _{mn} , \mat{c}\binary _{mn} \in \{ \mat{p}, \mat{q}, \mat{r}, -\mat{r} \} $. 


  \item 
    For each $ i \in [N_e] $ and $ \gamma \in [2N_eN_r] $,
    \begin{equation}
      \mat{a}_{i\gamma}\binary  =
      \begin{cases}
        \mat{p} & \text{if } (\gamma {\rm ~mod~}N_e) = (i {\rm ~mod~} N_e),\\
        \mat{q} & \text{otherwise.}
      \end{cases}
      \label{eq:vector-a}
    \end{equation}

  \item 
    Let
    $ \iota(\gamma) = (((\gamma - 1) \bmod 2 N_e) \bmod N_e ) + 1 $
    and
    $ \kappa(\gamma) = \lfloor (\gamma - 1) / 2 N_e \rfloor + 1 $.
    For each
    $ j \in [N_e] $ and $ \gamma \in [2N_e N_r] $,
    \begin{equation}
      \mat{b}_{j\gamma}\binary  =
      \begin{cases}
        \mat{p} & \text{if } x_{\iota(\gamma) j \kappa(\gamma)}=1,\\
        \mat{r} & \text{otherwise.}
      \end{cases}
      \label{eq:vector-b}
    \end{equation}

  \item 
    For each $ k \in [N_r] $ and $\gamma \in [2N_eN_r]$,
    \begin{equation}
      \mat{c}_{k\gamma}\binary  =
      \begin{cases}
        \mat{r} & \text{if } (\gamma  - 1) \bmod 2 N_e < N_e
        \text{ or } \left\lfloor ( \gamma - 1) / 2 N_e  \right\rfloor + 1 = k ,\\
        -\mat{r} &{\rm otherwise.}
      \end{cases}
      \label{eq:vector-c}
    \end{equation}
  \end{itemize}
\end{definition}


\begin{figure}
  \tiny

  \def\colwiz{1.9em}
    
  \begin{tikzpicture}
    \tikzset{every node/.style={outer sep=0pt}}
    \matrix [
      matrix of math nodes,
      left delimiter={[},
      right delimiter={]},
      column sep={\colwiz,between origins},
      row sep=.5ex
    ] (m) {
      \mat{p} & \mat{q} & \cdots & \mat{q} & \mat{p} & \mat{q} & \cdots & \mat{q} & [.3em]
      \mat{p} & \mat{q} & \cdots & \mat{q} & \mat{p} & \mat{q} & \cdots & \mat{q} & [.3em]
      \cdots & [.3em]
      \mat{p} & \mat{q} & \cdots & \mat{q} & \mat{p} & \mat{q} & \cdots & \mat{q} \\
      \mat{q} & \mat{p} & \cdots & \mat{q} & \mat{q} & \mat{p} & \cdots & \mat{q} & [.3em]
      \mat{q} & \mat{p} & \cdots & \mat{q} & \mat{q} & \mat{p} & \cdots & \mat{q} & [.3em]
      \cdots & [.3em]
      \mat{q} & \mat{p} & \cdots & \mat{q} & \mat{q} & \mat{p} & \cdots & \mat{q} \\[-2ex]
      \vdots  & \vdots  & \ddots & \vdots  & \vdots  & \vdots  & \ddots & \vdots  & [.3em] 
      \vdots  & \vdots  & \ddots & \vdots  & \vdots  & \vdots  & \ddots & \vdots  & [.3em]
      \cdots  & [.3em]
      \vdots  & \vdots  & \ddots & \vdots  & \vdots  & \vdots  & \ddots & \vdots \\
      \mat{q} & \mat{q} & \cdots & \mat{p} & \mat{q} & \mat{q} & \cdots & \mat{p} & [.3em] 
      \mat{q} & \mat{q} & \cdots & \mat{p} & \mat{q} & \mat{q} & \cdots & \mat{p} & [.3em]
      \cdots & [.3em]
      \mat{q} & \mat{q} & \cdots & \mat{p} & \mat{q} & \mat{q} & \cdots & \mat{p} \\
    };
    \draw ($0.5*(m-1-8.north)+0.5*(m-1-9.north)$) -- ($0.5*(m-4-8.south)+0.5*(m-4-9.south)$) ;
    \draw ($0.5*(m-1-16.north)+0.5*(m-1-17.north)$) -- ($0.5*(m-4-16.south)+0.5*(m-4-17.south)$) ;
    \draw ($0.5*(m-1-17.north)+0.5*(m-1-18.north)$) -- ($0.5*(m-4-17.south)+0.5*(m-4-18.south)$) ;

    \draw [dotted] ($0.5*(m-1-4.north)+0.5*(m-1-5.north)$) -- ($0.5*(m-4-4.south)+0.5*(m-4-5.south)$) ;
    \draw [dotted] ($0.5*(m-1-12.north)+0.5*(m-1-13.north)$) -- ($0.5*(m-4-12.south)+0.5*(m-4-13.south)$) ;
    \draw [dotted] ($0.5*(m-1-21.north)+0.5*(m-1-22.north)$) -- ($0.5*(m-4-21.south)+0.5*(m-4-22.south)$) ;
    
    \node[left=2.5em of m-1-1, anchor=center] (left-1) {$ \bm{a}_{1:}\binary $};
    \node[left=2.5em of m-2-1, anchor=center] (left-2) {$ \bm{a}_{2:}\binary $};
    \node[left=2.5em of m-3-1, anchor=center] (left-3) {$\vdots$};
    \node[left=2.5em of m-4-1, anchor=center] (left-4) {$ \bm{a}_{N_{e:}}\binary $};

    \node[rectangle,above delimiter=\{] (del-top) at ($ 0.5*(m-1-1) + 0.5*(m-1-4)$) {\tikz \path (0,0) rectangle (3.8 * \colwiz, 0); };
    \node[above=2ex] at (del-top.north) {$N_e$};

    \node[rectangle,above delimiter=\{] (del-top) at ($ 0.5*(m-1-5) + 0.5*(m-1-8)$) {\tikz \path (0,0) rectangle (3.8 * \colwiz, 0); };
    \node[above=2ex] at (del-top.north) {$N_e$};

    \node[rectangle,above delimiter=\{] (del-top) at ($ 0.5*(m-1-9) + 0.5*(m-1-12)$) {\tikz \path (0,0) rectangle (3.8 * \colwiz, 0); };
    \node[above=2ex] at (del-top.north) {$N_e$};

    \node[rectangle,above delimiter=\{] (del-top) at ($ 0.5*(m-1-13) + 0.5*(m-1-16)$) {\tikz \path (0,0) rectangle (3.8 * \colwiz, 0); };
    \node[above=2ex] at (del-top.north) {$N_e$};

    \node[rectangle,above delimiter=\{] (del-top) at ($ 0.5*(m-1-18) + 0.5*(m-1-21)$) {\tikz \path (0,0) rectangle (3.8 * \colwiz, 0); };
    \node[above=2ex] at (del-top.north) {$N_e$};

    \node[rectangle,above delimiter=\{] (del-top) at ($ 0.5*(m-1-22) + 0.5*(m-1-25)$) {\tikz \path (0,0) rectangle (3.8 * \colwiz, 0); };
    \node[above=2ex] at (del-top.north) {$N_e$};

    \node[rectangle, above delimiter=\{, yshift=5ex] (del-top) at ($ 0.5*(m-1-1) + 0.5*(m-1-25)$) {\tikz{\path (m-1-1.west) rectangle (m-1-25.east);} };
    \node[above=2ex] at (del-top.north) {$2 N_e N_r$};
  \end{tikzpicture}

  \begin{tikzpicture}
    \tikzset{every node/.style={outer sep=0pt}}
    \matrix [
      matrix of math nodes,
      left delimiter={[},
      right delimiter={]},
      column sep={\colwiz,between origins},
      row sep=.5ex
    ] (m) {
      \mat{r} & \mat{r} & \cdots & \mat{r} & \mat{r}  & \mat{r}  & \cdots & \mat{r}  & [.3em]
      \mat{r} & \mat{r} & \cdots & \mat{r} & -\mat{r} & -\mat{r} & \cdots & -\mat{r} & [.3em]
      \cdots  & [.3em]
      \mat{r} & \mat{r} & \cdots & \mat{r} & -\mat{r} & -\mat{r} & \cdots & -\mat{r} \\
      \mat{r} & \mat{r} & \cdots & \mat{r} & -\mat{r} & -\mat{r} & \cdots & -\mat{r} & [.3em]
      \mat{r} & \mat{r} & \cdots & \mat{r} & \mat{r}  & \mat{r}  & \cdots & \mat{r}  & [.3em]
      \cdots  & [.3em]
      \mat{r} & \mat{r} & \cdots & \mat{r} & -\mat{r} & -\mat{r} & \cdots & -\mat{r} \\[-2ex]
      \vdots  & \vdots  & \ddots & \vdots  & \vdots   & \vdots   & \ddots & \vdots   & [.3em]
      \vdots  & \vdots  & \ddots & \vdots  & \vdots   & \vdots   & \ddots & \vdots   & [.3em]
      \cdots  & [.3em]
      \vdots  & \vdots  & \ddots & \vdots  & \vdots   & \vdots   & \ddots & \vdots   \\
      \mat{r} & \mat{r} & \cdots & \mat{r} & -\mat{r} & -\mat{r} & \cdots & -\mat{r} & [.3em]
      \mat{r} & \mat{r} & \cdots & \mat{r} & -\mat{r} & -\mat{r} & \cdots & -\mat{r} & [.3em]
      \cdots  & [.3em]
      \mat{r} & \mat{r} & \cdots & \mat{r} & \mat{r}  & \mat{r}  & \cdots & \mat{r}  \\
    };
    \draw ($0.5*(m-1-8.north)+0.5*(m-1-9.north)$) -- ($0.5*(m-4-8.south)+0.5*(m-4-9.south)$) ;
    \draw ($0.5*(m-1-16.north)+0.5*(m-1-17.north)$) -- ($0.5*(m-4-16.south)+0.5*(m-4-17.south)$) ;
    \draw ($0.5*(m-1-17.north)+0.5*(m-1-18.north)$) -- ($0.5*(m-4-17.south)+0.5*(m-4-18.south)$) ;

    \draw [dotted] ($0.5*(m-1-4.north)+0.5*(m-1-5.north)$) -- ($0.5*(m-4-4.south)+0.5*(m-4-5.south)$) ;
    \draw [dotted] ($0.5*(m-1-12.north)+0.5*(m-1-13.north)$) -- ($0.5*(m-4-12.south)+0.5*(m-4-13.south)$) ;
    \draw [dotted] ($0.5*(m-1-21.north)+0.5*(m-1-22.north)$) -- ($0.5*(m-4-21.south)+0.5*(m-4-22.south)$) ;

    \node[left=2.5em of m-1-1, anchor=center] (left-1) {$ \bm{c}_{1:}\binary $};
    \node[left=2.5em of m-2-1, anchor=center] (left-2) {$ \bm{c}_{2:}\binary $};
    \node[left=2.5em of m-3-1, anchor=center] (left-3) {$\vdots$};
    \node[left=2.5em of m-4-1, anchor=center] (left-4) {$ \bm{c}_{N_{e:}}\binary $};

    \node[rectangle,above delimiter=\{] (del-top) at ($ 0.5*(m-1-1) + 0.5*(m-1-4)$) {\tikz \path (0,0) rectangle (3.8 * \colwiz, 0); };
    \node[above=2ex] at (del-top.north) {$N_e$};

    \node[rectangle,above delimiter=\{] (del-top) at ($ 0.5*(m-1-5) + 0.5*(m-1-8)$) {\tikz \path (0,0) rectangle (3.8 * \colwiz, 0); };
    \node[above=2ex] at (del-top.north) {$N_e$};

    \node[rectangle,above delimiter=\{] (del-top) at ($ 0.5*(m-1-9) + 0.5*(m-1-12)$) {\tikz \path (0,0) rectangle (3.8 * \colwiz, 0); };
    \node[above=2ex] at (del-top.north) {$N_e$};

    \node[rectangle,above delimiter=\{] (del-top) at ($ 0.5*(m-1-13) + 0.5*(m-1-16)$) {\tikz \path (0,0) rectangle (3.8 * \colwiz, 0); };
    \node[above=2ex] at (del-top.north) {$N_e$};

    \node[rectangle,above delimiter=\{] (del-top) at ($ 0.5*(m-1-18) + 0.5*(m-1-21)$) {\tikz \path (0,0) rectangle (3.8 * \colwiz, 0); };
    \node[above=2ex] at (del-top.north) {$N_e$};

    \node[rectangle,above delimiter=\{] (del-top) at ($ 0.5*(m-1-22) + 0.5*(m-1-25)$) {\tikz \path (0,0) rectangle (3.8 * \colwiz, 0); };
    \node[above=2ex] at (del-top.north) {$N_e$};

    \node[rectangle, above delimiter=\{, yshift=5ex] (del-top) at ($ 0.5*(m-1-1) + 0.5*(m-1-25)$) {\tikz{\path (m-1-1.west) rectangle (m-1-25.east);} };
    \node[above=2ex] at (del-top.north) {$2 N_e N_r$};

  \end{tikzpicture}

  \caption{Regularities in $\mat{A}\binary$ and $\mat{C}\binary$ of $\text{BCP}^*_\Delta (\mat{\mathcal{X}})$.
    Solid vertical lines separate ``pages,'' and dotted lines separate ``halfpages'' within each page.
    Matrix $\mat{B}\binary$ also shows some regularities; see Lemma~\ref{lem:faithful-encoding-periodicity}\ref{itm:b-periodicity}.
    It is not depicted here as its elements are dependent on $\mat{\mathcal{X}}$.
  }
\label{fig:general_factor_matrix}
\end{figure}



Figure~\ref{fig:general_factor_matrix}
illustrates
the binary factor matrices of
Definition~\ref{def:bcp-faithful-encoding}.
As seen from the figure,
each row of the matrices exhibits a certain pattern every $2N_e$ blocks, which we call a \emph{page}.
Each page can further be divided into a pair of \emph{halfpages} consisting of $N_e$ blocks each, which also show periodic patterns.
These patterns are stated precisely in Lemma~\ref{lem:faithful-encoding-periodicity} below.
It extensively uses addressing functions $ \alpha(k,m) = 2N_e(k-1) + m$ and $\beta(k,m) = \alpha(k,m) + N_e$
to specify the position of an individual block within the page it belongs to;
function $\alpha(k,m)$ designates the $m$th block within the first halfpage of the $k$th page,
whereas $\beta(k,m)$ designates the $m$th block in the second halfpage. 
There is a one-to-one correspondence between linear addressing by $ \gamma \in [2 N_e N_r]$
and 2-dimensional indexing by $k \in [N_r]$ and $m \in [N_e]$, combined with $\alpha$ and $\beta$ to specify a halfpage.


\begin{lemma}
  \label{lem:faithful-encoding-periodicity}
  Let $\mat{p}$, $\mat{q}$, and $\mat{r}$ be as given by Eqs.~\eqref{eq:p}--\eqref{eq:r}.
  The following statements \ref{itm:a-periodicity}--\ref{itm:c-periodicity-2} hold for
  block matrices
  $\mat{A}\binary  = [ \mat{a}\binary _{mn} ]$, $\mat{B}\binary  = [ \mat{b}\binary _{mn} ]$, $\mat{C}\binary  = [ \mat{c}\binary _{mn} ]$,
  where
  $ (\mat{A}\binary $, $\mat{B}\binary $, $\mat{C}\binary ) = \text{BCP}_\Delta^*(\mat{\mathcal{X}}) $
  is given by Definition~\ref{def:bcp-faithful-encoding}.
  \begin{enumerate}[label=(\alph*)]

  \item 
    \label{itm:a-periodicity}
    For any $i \in [N_e]$ and $ \gamma \in [ (2 N_r - 1) N_e ] $,
    \begin{equation}
      \mat{a}\binary _{i\gamma} = \mat{a}\binary _{i,(\gamma + N_e)} .
      \label{eq:a-periodicity}
    \end{equation}
    In particular, for any $ i, m \in [N_e] $ and $ k \in [N_r] $,
    \begin{equation}
      \mat{a}\binary _{i\alpha(k,m)} = \mat{a}\binary _{i \beta(k,m)} .
      \label{eq:a-periodicity-alpha-beta}
    \end{equation}

  \item 
    \label{itm:a-alpha}
    For any $i \in [N_e]$ and $ k \in [N_r] $,
    \begin{equation}
      \mat{a}\binary _{i\alpha(k,i)} = \mat{p} .
      \label{eq:a-alpha}
    \end{equation}

  \item 
    \label{itm:a-non-alpha}
    For any $ i \in [N_e] $, $ m' \in [N_e] \backslash \{ i \} $, and $ k \in [N_r] $,
    \begin{equation}
      \mat{a}\binary _{i\alpha(k,m')} = \mat{q} .
      \label{eq:a-non-alpha}
    \end{equation}

  \item 
    \label{itm:b-periodicity}
    For any $j, m \in [N_e]$ and $k \in [N_r]$,
    \begin{equation}
      \mat{b}\binary _{ j\alpha(k,m) } = \mat{b}\binary _{j\beta(k,m) } .
      \label{eq:b-periodicity}
    \end{equation}

  \item 
    \label{itm:c-periodicity}
    For any $k \in [N_r]$ and $ m \in [ N_e ] $,
    \begin{equation}
      \mat{c}_{k \alpha(k,m)}\binary  = \mat{c}_{k \beta(k,m)}\binary  = \mat{r}.
      \label{eq:c-periodicity}
    \end{equation}

  \item 
    \label{itm:c-periodicity-2}
    For any $k \in [N_r], m \in [N_e]$ and $n^\prime \in [N_r] \backslash \{ k\}$,
    \begin{equation}
      \mat{c}_{k \alpha(n^\prime, m)}\binary  = -\mat{c}_{k \beta(n^\prime, m)}\binary  .
      \label{eq:c-periodicity-2}
    \end{equation}
  \end{enumerate}
\end{lemma}


\begin{proof}\leavevmode
  \begin{enumerate}[label=(\alph*)]
  \item 
    The statement follows from $(\gamma \bmod N_e) = ((\gamma + N_e) \bmod N_e)$ and the definition of $\mat{a}\binary $ given by Eq.~\eqref{eq:vector-a}.
  \item 
    Follows from $\alpha(k, i) \bmod N_e = i$ and Eq.~\eqref{eq:vector-a}.
  \item 
    Follows from $\alpha(k, m') \bmod N_e \ne i$ (because $m' \in [N_e]\backslash \{ i \}$) and Eq.~\eqref{eq:vector-a}.
  \item 
    Follows from
    $m = \iota(\alpha(n, m)) = \iota(\beta(n, m)) $ and $n = \kappa(\alpha(n, m)) = \kappa(\beta(n, m))$.
  \item 
    Follows from Eq.~\eqref{eq:vector-c}.
  \item 
    Follows from
    Eq. \eqref{eq:vector-c},
    specifically,
    $\mat{c}_{k\alpha(n^\prime, m)}\binary  = \mat{r}$
    and
    $\mat{c}_{k\beta(n^\prime, m)}\binary = -\mat{r}$.
    \qedhere
  \end{enumerate}
\end{proof}


The following corollary is a direct consequence of Lemma~\ref{lem:faithful-encoding-periodicity}.

\begin{corollary}
  \label{cor:alpha-beta-terms-are-equal}
  For any $i, j \in [N_e]$ and $ k \in [N_r] $,
  \begin{equation*}
    ( \mat{a}_{i\alpha(k,i)}\binary  \circ \mat{b}_{j\alpha(k,i)}\binary  ) \mat{c}_{k\alpha(k,i)}\transposeb
    =
    ( \mat{a}_{i\beta (k,i)}\binary  \circ \mat{b}_{j\beta (k,i)}\binary  ) \mat{c}_{k\beta (k,i)}\transposeb .
  \end{equation*}
\end{corollary}


From Lemma~\ref{lem:faithful-encoding-periodicity}\ref{itm:a-periodicity}, \ref{itm:b-periodicity}, and \ref{itm:c-periodicity-2}, we also have the following.

\begin{corollary}
  \label{cor:zero-sum-abc}
  For any $i, j, m \in [N_e]$, $k \in [N_r]$, and $n^\prime \in [N_r] \backslash \{k\}$,
  \begin{equation*}
    (\mat{a}_{i\alpha(n^\prime, m)}\binary  \circ \mat{b}_{j\alpha(n^\prime, m)}\binary ) \mat{c}_{k\alpha(n^\prime, m)}\transposeb
    +
    (\mat{a}_{i\beta (n^\prime, m)}\binary  \circ \mat{b}_{j\beta (n^\prime, m)}\binary ) \mat{c}_{k\beta (n^\prime, m)}\transposeb
    = 0 .
  \end{equation*}
\end{corollary}


These lemmas and corollaries lead to the following:

\begin{lemma}
  \label{lem:sum-of-remaining-terms-are-null}
  For any $ i, j \in [N_e] $ and $ k \in [N_r] $,
  \begin{equation*}
    \sum_{ \gamma \in [2 N_e N_r] \backslash \{ \alpha(k,i), \beta(k,i) \} }
    \!\!\!\!\!\!\!\!\!\!\!\!
    ( \mat{a}_{i\gamma}\binary  \circ \mat{b}_{j\gamma}\binary )\mat{c}_{k\gamma}\transposeb
    = 0 .
  \end{equation*}
\end{lemma}
\begin{proof}
  Let $\mat{p}$, $\mat{q}$, and $\mat{r}$ be as given by Eqs.~\eqref{eq:p}--\eqref{eq:r}.
  From Lemma~\ref{lem:faithful-encoding-periodicity}\ref{itm:a-periodicity}, \ref{itm:a-non-alpha}, \ref{itm:b-periodicity}, and \ref{itm:c-periodicity},
  we have for any $m^\prime \in [N_e] \backslash \{i\}$,
  \begin{align*}
    \mat{a}_{i,\alpha(k, m^\prime)}\binary  & = \mat{a}_{i,\beta(k, m^\prime)}\binary  = \mat{q} , \\
    \mat{b}_{j,\alpha(k, m^\prime)}\binary  & = \mat{b}_{j,\beta(k, m^\prime)}\binary  , \\
    \mat{c}_{k,\alpha(k, m^\prime)}\binary  & = \mat{c}_{k,\beta(k, m^\prime)}\binary  = \mat{r} .
  \end{align*}

  Thus, for any $m^\prime \in [N_e] \backslash \{i\}$,
  \begin{align}
    (     \mat{a}_{ i\alpha(k, m^\prime) }\binary  \circ \mat{b}_{ j \alpha(k, m^\prime)}\binary  ) \mat{c}_{ k \alpha(k, m^\prime) }\transposeb
    & = ( \mat{a}_{ i\beta (k, m^\prime) }\binary  \circ \mat{b}_{ j \beta (k, m^\prime)}\binary  ) \mat{c}_{ k \beta (k, m^\prime) }\transposeb  \nonumber \\
    & = ( \mat{q} \circ \mat{b}_{ j \alpha(k, m^\prime)}\binary  )  \mat{r} \transpose  \nonumber \\
    & = 0.
      \label{eq:zero-abc}
  \end{align}
  The last equality holds because $\mat{b}_{ j \alpha(k, m^\prime)}\binary $ is either $\mat{p}$ or $\mat{r}$ by definition,
  and $ ( \mat{q} \circ \mat{p} )  \mat{r} \transpose  = ( \mat{q} \circ \mat{r} )  \mat{r} \transpose  = 0 $.

  Now,
  for any $i, j \in [N_e]$ and $ k \in [N_r] $,
  \begin{align*}
    \text{\rlap{$ \sum_{ \gamma \in [2 N_e N_r] \backslash \{ \alpha(k,i), \beta(k,i) \} } \!\!\!\!\!\!\!\!
    ( \mat{a}_{i\gamma}\binary \circ \mat{b}_{j\gamma}\binary )\mat{c}_{k\gamma}\transposeb $}} \nonumber \\
    \quad & =
            \!\!\!\!
            \sum_{ \substack{ n^\prime \in [N_r]\backslash\{k\} \\ m \in [N_e] } }
            \underbrace{
              \left(
                ( \mat{a}_{i\alpha(n^\prime, m)}\binary  \circ \mat{b}_{j\alpha(n^\prime, m)}\binary  ) \mat{c}_{k\alpha(n^\prime, m)}\transposeb
                +
                ( \mat{a}_{i\beta (n^\prime, m)}\binary  \circ \mat{b}_{j\beta (n^\prime, m)}\binary  ) \mat{c}_{k\beta (n^\prime, m)}\transposeb
              \right)
            }_{ = 0 \; \text{\rlap{ by Cororally~\ref{cor:zero-sum-abc}}} }
            \\
          & \quad +
            \!\!\!\!
            \sum_{m^\prime \in [N_e] \backslash \{i\}}
            \underbrace{
              \left(
                (\mat{a}_{i\alpha(k, m^\prime)}\binary  \circ \mat{b}_{j\alpha(k, m^\prime)}\binary ) \mat{c}_{k\alpha(k, m^\prime)}\transposeb
                +
                (\mat{a}_{i\beta (k, m^\prime)}\binary  \circ \mat{b}_{j\beta (k, m^\prime)}\binary ) \mat{c}_{k\beta (k, m^\prime)}\transposeb
              \right)
            }_{ = 0 \; \text{\rlap{ by Eq.~\eqref{eq:zero-abc}}} }
            \\
          & = 0. 
            \qedhere
  \end{align*}
\end{proof}


We are now ready to prove Theorem~\ref{thm:expressiveness}, restated here as Theorem~\ref{thm:expressiveness-restatement}.


\begin{thm}[Theorem~\ref{thm:expressiveness}; full expressiveness of B-CP]
  \label{thm:expressiveness-restatement}
  For an arbitrary binary tensor
  $\mat{\mathcal{X}} \in \{0, 1\}^{N_e \times N_e \times N_r}$,
  there exists a B-CP decomposition with binary factor matrices
  $ \mat{A}\binary  , \mat{B}\binary  \in \{+\Delta, -\Delta\}^{N_e \times D} $  and $\mat{C}\binary  \in \{+\Delta, -\Delta\}^{N_r \times D} $
  for some $D$ and $\Delta$,
  such that
  \begin{equation}
    \mat{\mathcal{X}} = \sum_{ d \in [D] } \mat{a}\binary _d \otimes \mat{b}\binary _d \otimes \mat{c}\binary _d .
    \label{eq:b-cp-model}
  \end{equation}
\end{thm}


\begin{proof}
  Let
  $\Delta = 1/2$
  and let
  $ (\mat{A}\binary , \mat{B}\binary , \mat{C}\binary ) = \text{BCP}_\Delta^*(\mat{\mathcal{X}}) $
  given by Definition~\ref{def:bcp-faithful-encoding}.
  We show that these matrices indeed satisfy Eq.~\eqref{eq:b-cp-model}.

  For any $i, j \in [N_e]$ and $ k \in [N_r] $,
  the score $\theta_{i j k}$ for triple $(e_i, e_j, r_k)$ is:
  \begin{align}
    \theta_{ijk} & = ( \mat{a}_{i:}\binary \circ \mat{b}_{j:}\binary )\mat{c}_{k:}\transposeb \nonumber \\
                 & = \sum_{\gamma \in [2 N_e N_r]} (\mat{a}_{i\gamma}\binary \circ \mat{b}_{j\gamma}\binary )\mat{c}_{k\gamma}\transposeb \nonumber \\
                 & = (\mat{a}_{i\alpha(k,i)}\binary \circ \mat{b}_{j\alpha(k,i)}\binary )\mat{c}_{k\alpha(k,i)}\transposeb \nonumber \\
                 & \qquad +
                   \underbrace{
                   (\mat{a}_{i\beta(k,i)}\binary \circ \mat{b}_{j\beta(k,i)}\binary )\mat{c}_{k\beta(k,i)}\transposeb
                   }_{
                   =(\mat{a}_{i\alpha(k,i)}\binary \circ \mat{b}_{j\alpha(k,i)}\binary )\mat{c}_{k\alpha(k,i)}\transposeb \text{ by Corollary~\ref{cor:alpha-beta-terms-are-equal}}
                   }
                   +
                   \underbrace{
                   \sum_{\gamma \neq \alpha(k,i), \beta(k,i)} \!\!\!\! (\mat{a}_{i\gamma}\binary \circ \mat{b}_{j\gamma}\binary )\mat{c}_{k\gamma}\transposeb
                   }_{
                   =0 \text{ by Lemma~\ref{lem:sum-of-remaining-terms-are-null}}
                   } \nonumber \\
                 & = 2(\mat{a}_{i\alpha(k,i)}\binary \circ \mat{b}_{j\alpha(k,i)}\binary ) \mat{c}_{k\alpha(k,i)}\transposeb . \label{eq:score-breakdown}
  \end{align}
  Let $\mat{p}, \mat{q}, \mat{r} \in \{ +\Delta, -\Delta \}^{1\times 4}$ be as given by Eqs.~\eqref{eq:p}--\eqref{eq:r}.
  For any $i, j \in [N_e]$ and $ k \in [N_r] $,
  if $x_{ijk}=1$,
  $\mat{b}_{j\alpha(k,i)}\binary  = \mat{p}$
  by Eq.~\eqref{eq:vector-b},
  and the following equation holds:
  \begin{align}
    \theta_{ijk}
        & = 2( \mat{a}_{i\alpha(k,i)}\binary \circ \mat{b}_{j\alpha(k,i)}\binary  ) \mat{c}_{k\alpha(k,i)}\transposeb &  & \because \text{ Eq.~\eqref{eq:score-breakdown}} \nonumber                                                          \\
        & = 2( \mat{p} \circ \mat{p}) \mat{r}\transpose                                                               &  & \because \text{ Lemma~\ref{lem:faithful-encoding-periodicity}\ref{itm:a-alpha}, \ref{itm:c-periodicity}} \nonumber \\
        & = 8\Delta^3  \nonumber                                                                                                                                                                                                              \\
        & = 1 \; (= x_{ijk}).                                                                                         &  & \because \Delta = 1/2 \label{eq:sign1}
  \end{align}
  And if $x_{ijk}=0$, 
  $\mat{b}_{j\alpha(k,i)}\binary  = \mat{r}$,
  and we have:
  \begin{align}
    \theta_{ijk}
        & = 2(\mat{a}_{i\alpha(k,i)}\binary \circ \mat{b}_{j\alpha(k,i)}\binary ) \mat{c}_{k\alpha(k,i)}\transposeb   &  & \because \text{ Eq.~\eqref{eq:score-breakdown}} \nonumber                                                          \\
        & = 2(\mat{p} \circ \mat{r}) \mat{r}\transpose                                                                &  & \because \text{ Lemma~\ref{lem:faithful-encoding-periodicity}\ref{itm:a-alpha}, \ref{itm:c-periodicity}} \nonumber \\
        & = 0 \; (= x_{ijk}) . \label{eq:sign2}
  \end{align}
  By Eqs.~\eqref{eq:sign1} and \eqref{eq:sign2}, $ x_{ijk} = \theta_{ijk} $
  holds irrespective of the value of $x_{ijk}$.
  Hence,
  $
  \mat{\mathcal{X}}
  =
  \sum_{d \in [D]} \mat{a}_{d}\binary  \otimes \mat{b}_{d}\binary  \otimes \mat{c}_{d}\binary 
  $
  where 
  $D = 8 N_e N_r$.
\end{proof}




\bibliographystyle{elsarticle-harv}
\bibliography{ref}


\end{document}